\documentclass[letterpaper, 10 pt, conference]{ieeeconf}  

\IEEEoverridecommandlockouts                              

\usepackage{graphics} 
\usepackage{epsfig} 
\usepackage{mathptmx} 
\usepackage{times} 
\usepackage{amsmath} 
\usepackage{amssymb}  
\usepackage{subcaption} 
\usepackage[linesnumbered,ruled]{algorithm2e}
\usepackage{csquotes}
\usepackage{multicol}
\usepackage{amsfonts}
\usepackage{hyperref}  

\newtheorem{theorem}{Theorem}
\newtheorem{lemma}{Lemma}
\newtheorem{assumption}{Assumption}

\title{\LARGE \bf
Modified Meta-Thompson Sampling for Linear Bandits and  Its Bayes Regret Analysis}

\author{Hao Li, Dong Liang$^*$ and Zheng Xie
\thanks{$*$ Corresponding author}
\thanks{Hao Li, Dong Liang and Zheng Xie are with College of Sciences, National University of Defense Technology, Changsha, Hunan, China, 410073
        {\tt\small lihaomath@163.com(Hao Li), dongliangnudt@nudt.edu.cn(Dong Liang), xiezheng81@nudt.edu.cn(Zheng Xie)}}%
}

\begin{document}

\maketitle
\thispagestyle{empty}
\pagestyle{empty}

\begin{abstract}
Meta-learning is characterized by its ability to learn how to learn, enabling the adaptation of learning strategies across different tasks.
Recent research introduced the Meta-Thompson Sampling (Meta-TS), which meta-learns an unknown prior distribution   sampled from a meta-prior  by interacting with bandit instances drawn from it. However, its analysis was limited to Gaussian bandit. The contextual multi-armed bandit  framework   is an extension of the Gaussian Bandit, which challenges agent to utilize context vectors to predict the most valuable arms, optimally balancing exploration and exploitation to minimize regret over time. This paper introduces Meta-TSLB algorithm, a modified Meta-TS for linear contextual bandits. We theoretically analyze Meta-TSLB and derive an $ O\left( \left( m+\log \left( m \right) \right) \sqrt{n\log \left( n \right)} \right)$  bound on its Bayes regret, in which  $m$ represents the number of bandit instances, and $n$  the number of rounds of Thompson Sampling. Additionally, our work complements the analysis of Meta-TS for linear contextual bandits. The performance of Meta-TSLB is evaluated experimentally under different settings, and we experimente and analyze the  generalization capability of Meta-TSLB, showcasing its potential to adapt to unseen instances. 

\end{abstract}

\section{Introduction}

The multi-armed bandit  framework encapsulates the fundamental exploration-exploitation dilemma prevalent in sequential decision-making scenarios. Among its various versions, the contextual multi-armed bandit problem stands out. In this setting, an agent confronts $n$ rounds, each presenting a choice from $k$ distinct actions, or arms. Prior to selecting an arm, the agent is privy to  $k$ context vectors, also designated as feature vectors, associated with $k$ arms.
Leveraging these context vectors in conjunction with historical reward from previously pulled arms, the agent decides which arm to pull in the current round. 
The objective is to progressively unravel the intricate relationship between context vectors and rewards, thereby enabling precise predictions of the most rewarding arm based solely on contextual information.


In the realm of contextual bandits with linear payoff functions, the agent engages in a competition with the class of all  linear  predictors on the context vectors. We focus on a stochastic contextual bandit problem under the assumption of linear realizability, postulating the existence of an unknown parameter $\boldsymbol{\mu } \in \mathbb{R}^d$
such that the expected reward for arm $i$ given context  $\boldsymbol{b}_i$
is $\boldsymbol{b}_i^T\boldsymbol{\mu }$.
Under this realizability assumption, the optimal predictor is inherently the linear predictor associated with $\boldsymbol{\mu }$, and the agent's goal is to learn this underlying parameter. In this work, we concisely refer to this problem as Linear Bandits. This realizability assumption aligns with established practices in the literature on contextual multi-armed bandits, as exemplified in \cite{02, L03, CB, L04}.

In 2021, Kveton et al. \cite{metaTS} delved into a general framework where the agent can prescribe an uncertain prior, with the true prior learned through sequential interactions with bandit instances.  Specifically,  a learning agent sequentially interacts with $m$ bandit instances, each represented by a parameter $\boldsymbol{\mu }$. Each such interaction, spanning $n$ rounds, is termed a \textit{task}. These instances share a commonality: their parameters are drawn from an unknown instance prior $P_*$, which is sampled from a known \textit{meta-prior} $Q$.
The agent's objective is to minimize regret in each sampled instance, performing almost as if $P_*$ were known. This is accomplished by adapting to $P_*$ through the instances interactions, embodying a form of meta-learning \cite{ML1,ML2,ML3,ML4}.
To tackle this challenge, Kveton et al. \cite{metaTS} introduced Meta-TS, a meta-Thompson sampling algorithm. This algorithm is designed to bound the Bayes regret in Gaussian bandits, demonstrating its efficacy in leveraging meta-learning to optimize performance across bandit instances. Subsequently, Azizi et al. \cite{azizi} devised a meta-learning framework aimed at minimizing simple regret in bandit problems. A key application of meta-prior exploration in recommender systems is assessing users' latent interests for items like movies. Each user is a bandit instance, with items as arms. While a standard prior could aid Thompson sampling, the true form may be uncertain. This work formalizes learning a prior over user interests to explore preferences of new ("cold-start") users.


In this paper, we improve  Meta-TS  proposed by Kveton et al. \cite{metaTS} to address the linear bandit problem, naming it Meta-TSLB. We provide a theoretical analysis of Meta-TSLB, deriving an upper bound on its Bayes regret. Moreover, as the   Bayes regret analysis of Meta-TS \cite{metaTS} was confined to Gaussian bandits, we also establish an upper bound for Meta-TS when applied to linear bandits.
Furthermore, we explore various formulations of linear bandits and evaluate their performance empirically through experiments. Finally, we demonstrate and analyze the generalization ability of the algorithm through experiments.

\section{Problem Statement}

Throughout this paper, we use $[ n]$ to denote the set  $\{1,\cdots,n\}$, $\left\|\cdot \right\| $ to represent the $l_2$-norm  and  $\mathbf{1}\left\{ E   \right\}$ to denote that event $E$ occurs.  
A linear bandit instance  with $k$ arms  is parameterized by an unknown $\boldsymbol{\mu } \in \mathbb{R}^d$ to the agent. The reward $r_i\left( t \right)$  of arm $i$ in round  $t$ in instance $\boldsymbol{\mu }$ is drawn i.i.d. from a Gaussian distribution $\mathcal{N}( \boldsymbol{b }_i\left( t \right) ^T\boldsymbol{\mu },v^2)$, where $ \boldsymbol{b }_i\left( t \right) ^T\boldsymbol{\mu }$ is the mean, $v^2$ is the variance and 
$\boldsymbol{b }_i\left( t \right) $ is a context vector of arm $i$  in round $t$.
Suppose that  $\left\| \boldsymbol{b}_i\left( t \right) \right\| \leqslant 1$.

The agent engages in a sequential interaction with $m$ bandit instances, each uniquely identified by an index $s\in [ m]$. We designate each of these interactions as a distinct task. While acknowledging that, in practice, the context vectors of individual instances may vary, we adopt a simplifying assumption that all instances share a common set of context vectors. This assumption is purely for notational convenience and does not alter the subsequent analytical process or conclusions drawn from it.

The problem in a Bayesian  fashion is formalized as below.  Assume the availability of a prior distribution $Q=\mathcal{N}(\boldsymbol{\mu }_Q,v^2 \varSigma _Q )$, and the instance prior is set as $P_*=\mathcal{N}(\boldsymbol{\mu }_*,v^2 \varSigma _* )$ where $\boldsymbol{\mu }_*\sim Q$. In this work, we do not require that $\varSigma _Q, \varSigma _*$ are diagonal matrices. We refer to $Q$ as a meta-prior since it is a prior over priors. The agent knows $Q$ and  covariance matrix   $v^2 \varSigma _*$ in $P_*$  but not $\boldsymbol{\mu }_*$. At the beginning of task $s\in[m]$, an instance $\boldsymbol{\mu }_{s}$ is sampled i.i.d. from  $P_*$.  The agent interacts with  $\boldsymbol{\mu }_{s}$  for $n$ rounds to learn $\boldsymbol{\mu }_*$.

Denote by $A_{s,t}\in[k]$ the pulled arm in round $t$ of task $s$. The result of the interactions in a task $s$ to round $t$  is history
$$
    H_{s,t}=\left\{ A_{s,\tau},r_{A_{s,\tau}}\left( \tau \right) ,\boldsymbol{b}_i\left( \tau \right) ,i=1,\cdots ,k,\tau =1,\cdots ,t \right\} .
$$
We denote  by $H_{1:s}=H_{1,n}\otimes \cdots \otimes H_{s,n}$ the histories of tasks $1$ to $s$. When the agent needs to choose the pulled arm in the   round  $t+1$ of task $s$, it can  observe $H_{1:s-1}$, $H_{s,t}$ and     the contexts $\boldsymbol{b}_i\left( t+1 \right),\ i\in[k] $.

The $n$-round Bayes regret of a learning agent or algorithm over $m$ tasks with instance prior $P_*$ is
$$
    R(m,n;P_*)=\sum_{s=1}^m{\mathbb{E}}\left[ \left.  \sum_{t=1}^n{\left( \boldsymbol{b}_{A^{*}_{s,t}}\left( t \right) -\boldsymbol{b}_{A_{_{s,t}}}\left( t \right) \right)}^T\boldsymbol{\mu }_s\right|\:P_* \right] ,
$$
where $A_{_{s,t}}^{*}=\mathrm{arg} \max_{i\in [k]} \,\,\boldsymbol{b}_i\left( t \right) ^T\boldsymbol{\mu }_s$ is the optimal arm in the bandit instance $\boldsymbol{\mu }_s$ in round $t$ of  task $s$. The above expectation is over bandit instances $\boldsymbol{\mu }_s$ sampled from $P_*$, their realized rewards, and pulled arms.


\section{ Modified Meta-Thompson Sampling for Linear Bandits}

Thompson Sampling (TS) \cite{TS1,TS2,TS3,TS4} stands as the premier and widely adopted bandit algorithm, parameterized by   a  specified prior. Kveton et al. \cite{metaTS} studied a more general setting  where the agent confronts uncertainty over an unknown prior $P_*$  and introduced Meta-Thompson Sampling (Meta-TS), a novel method using sequential interactions with randomly drawn bandit instances from $P_*$ to meta-learn the unknown $P_*$.


Agrawal et al. \cite{TSCB} generalized Thompson Sampling   for linear bandits by employing a Gaussian likelihood function and a standard multivariate Gaussian prior. We extend this algorithm by allowing any multivariate Gaussian distribution as the prior (Algorithm \ref{alg2}). When TS is applied to a bandit instance $\boldsymbol{\mu }$, the agent samples
$\tilde{\boldsymbol{\mu }}(t)$ from the posterior distribution $P(t)$ in round $t$ and selects the arm $i$ that maximizes $b_i(t)^T\tilde{\boldsymbol{\boldsymbol{\mu } }}(t)$. Subsequently, the agent receives a random reward $r_{A_t}\left( t \right) \sim \mathcal{N} \left( \boldsymbol{b }_{A_t}\left( t \right) ^T\boldsymbol{\mu } ,v^2 \right)$  and updates the posterior distribution to $P(t+1)$.

\begin{algorithm}[h]
    {\bf {Setting:}}  Bandit instance $\boldsymbol{\mu }$; 

    {\bf {Input:}}  Prior  $P(1)=\mathcal{N} \left( \hat{\boldsymbol{\mu }}\left( 1 \right) ,v^2B\left( 1 \right) ^{-1} \right)$

    {\bf {For}} $t=1,2,\cdots,n$ {\bf {do}}

    \qquad Sample  $\tilde{\boldsymbol{\mu }}(t) \sim  P(t)= \mathcal{N} \left( \hat{\boldsymbol{\mu }}\left( t \right) ,v^2B\left( t \right) ^{-1} \right)$

    \qquad Pull arm $
        A_t=\mathrm{arg} \underset{i}{\max}\,\,\boldsymbol{b }_i\left( t \right) ^T\tilde{\boldsymbol{\mu }}\left( t \right) $ and observe $r_{A_t}\left( t \right) \sim \mathcal{N} \left( \boldsymbol{b }_{A_t}\left( t \right) ^T\boldsymbol{\mu } ,v^2 \right)
    $

    \qquad Update $P(t)$ to $P(t+1)$

    {\bf {End For}}
    \caption{Thompson Sampling  for Linear Bandits (TS).}
    \label{alg2}
\end{algorithm}

The posterior distribution $P(t)= \mathcal{N} \left( \hat{\boldsymbol{\mu }}\left( t \right) ,v^2B\left( t \right) ^{-1} \right)$   is obtained by
   $ B\left( t \right) =B\left( 1 \right) +\sum_{\tau =1}^{t-1}{\boldsymbol{b}_{A_{\tau}}\left( \tau \right) \boldsymbol{b}_{A_{\tau}}\left( \tau \right) ^T}$ and $
    \hat{\boldsymbol{\mu}}\left( t \right) =B\left( t \right) ^{-1}\left[ B\left( 1 \right) \hat{\boldsymbol{\mu}}\left( 1 \right) +\sum_{\tau =1}^{t-1}{\boldsymbol{b}_{A_{\tau}}\left( \tau \right) r_{A_{\tau}}\left( \tau \right)} \right] .$

Meta-TSLB is a derivative of  Meta-TS  and also  formulates uncertainty over  $P_*$. This uncertainty is encapsulated within a meta-posterior, a probabilistic framework that spans potential instance priors.
We denote the meta-posterior in task $s$ by $Q_s=\mathcal{N}(\boldsymbol{\mu }_{Q,s},v^2 \varSigma _{Q,s} ) $, in which $\boldsymbol{\mu }_{Q,s}\in\mathbb{R}^d$ is an estimate of $\boldsymbol{\mu }_*$.
For task $s$, Meta-TSLB applies TS with prior $P_s=\mathcal{N}({\boldsymbol{\mu }}_{Q,s},v^2 \varSigma _* )$  to bandit instance $\boldsymbol{\mu }_{s}$ for $n$ rounds.
After that, the meta-posterior is updated by Lemma \ref{Qs}. The pseudocode of  Meta-TSLB    is presented in Algorithm \ref{alg1}. The difference between Meta-TSLB and  Meta-TS \cite{metaTS} lies in that Meta-TS samples $\hat{\boldsymbol{\mu}}_s$ from $Q_s$ and uses $P_s=\mathcal{N}(\hat{\boldsymbol{\mu}}_s,v^2 \varSigma _* )$ as the prior for the   task $s$.



\begin{algorithm}[h]
    {\bf {Setting:}} Instance prior   $P_*=\mathcal{N}(\boldsymbol{\mu }_*,v^2 \varSigma _* )$

    {\bf {Input:}} Meta-prior $Q=\mathcal{N}(\boldsymbol{\mu }_Q,v^2 \varSigma _Q ) $

    Compute $Q_1 \gets Q$

    {\bf {For}} $s=1,2,\cdots,m$ {\bf {do}}

    \qquad Sample bandit instance $\boldsymbol{\mu }_s \sim P_* $


    \qquad Apply TS (Algorithm \ref{alg2}) with prior  $P_s=\mathcal{N}({\boldsymbol{\mu }}_{Q,s},v^2 \varSigma _* )$  to $\boldsymbol{\mu }_s$ for $n$ rounds

    \qquad Update meta-posterior $Q_{s+1}  $ by Lemma \ref{Qs}

    {\bf {End For}}

    \caption{Modified Meta-TS for  Linear Bandits (Meta-TSLB).}
    \label{alg1}
\end{algorithm}


\begin{lemma}\label{Qs}
    In task $s$, the pulled arm in round $t$ of TS is $A_t$ and $I_d$ is $d$-dimensional unit matrix. 
    Then the meta-posterior in task $s+1$ is $Q_{s+1}=\mathcal{N}(\boldsymbol{\mu }_{Q,s+1},v^2 \varSigma _{Q,s+1} )$, where
    $
        \varSigma _{Q,s+1}=\left[ \varSigma _{Q,s}^{-1}+\varSigma _{*}^{-1}-\left( \varSigma _*\left( \sum_{t=1}^n{\boldsymbol{b }_{A_t}\left( t \right) \boldsymbol{b }_{A_t}\left( t \right) ^T} \right) \varSigma _*+\varSigma _* \right) ^{-1} \right] ^{-1},
    $ and 
    $
        \boldsymbol{\mu } _{Q,s+1}=\varSigma _{Q,s+1}\left[ {\boldsymbol{\mu } _{Q,s}}^T\varSigma _{Q,s}^{-1}+\left( \sum_{t=1}^n{r_{A_t}\left( t \right) \boldsymbol{b }_{A_t}\left( t \right)} \right) ^T\right.$
        $\left.\left( \varSigma _*\sum_{t=1}^n{\boldsymbol{b }_{A_t}\left( t \right) \boldsymbol{b }_{A_t}\left( t \right) ^T}+I_d \right) ^{-1} \right] ^T.
    $
\end{lemma}

\section{Regret Bound of Meta-TSLB}

In this section, we first conduct an analysis of the fundamental properties of TS. Subsequently, leveraging these properties, we present the most important theorem of this paper (Theorem \ref{theo_mainG}), which establishes the Bayes regret bound for Meta-TSLB. In addition, we supplement the Bayes regret bound of meta-TS applied to linear bandits (Theorem \ref{theo_main}). For a detailed and comprehensive proof of the process, kindly refer to Appendix.


In this paper,  for matrix $A$, we use $\lambda _{\max}(A)$ and $\lambda _{\min}(A)$ to represent the maximum and minimum eigenvalues of $A$, respectively. In particular, $\lambda _{\max}$ and $\lambda _{\min}$ specifically refer to the maximum and minimum eigenvalues of $\varSigma _*^{-1} $ in $P_*$.

\subsection{Key Lemmas}
Firstly, we present a assumption and delve into its significance, followed by a discussion on how to select the constant $\vartheta $. 
Using  $\vartheta $, we can greatly reduce the regret bound given in next subsection.

\begin{assumption}\label{con1}
 The inequality $\lambda _{\min}\left( B\left( t \right) \right)\geqslant 4\vartheta (t-1)$ holds  for  any  $A_{t}\in \left[ k \right]$, in which $\vartheta  $ is a positive constant.
\end{assumption}

$\vartheta  $ is a parameter related to the context vectors and $B(1)$. 
For symmetric positive definite matrices $A$ and $B$, we have
    \begin{equation}\label{eq18}
    \begin{aligned}
                        \lambda _{\min}\left( A+B \right) =\underset{\left\| \boldsymbol{x} \right\| \leqslant 1}{\min}\,\,\boldsymbol{x}^T\left( A+B \right) \boldsymbol{x} &\geqslant \underset{\left\| \boldsymbol{x} \right\| \leqslant 1}{\min}\,\,\boldsymbol{x}^TA\boldsymbol{x}+\underset{\left\| \boldsymbol{x} \right\| \leqslant 1}{\min}\,\,\boldsymbol{x}^TB\boldsymbol{x}\\&= \lambda _{\min}\left( A \right) + \lambda _{\min}\left( B \right) .
    \end{aligned}
    \end{equation}
Thus, $\lambda _{\min}\left( B\left( t \right)\right) $ is monotonically increasing with respect to $t$. Now, we present a way to choose $\vartheta  $. Assuming there exists a constant $\varDelta$  such that the matrix $
\sum_{\tau =t_0+1}^{t_0+\varDelta}{\boldsymbol{b}_{A_{\tau}}\left( \tau \right) \boldsymbol{b}_{A_{\tau}}\left( \tau \right) ^T}
$ is full rank  for  any $t_0\in \left[ n-\varDelta \right] $  and  $
\left\{ A_{\tau}\in \left[ k \right] \right\} _{\tau =t_0+1}^{t_0+\varDelta}
$.
Let $
\rho _{\min}=\underset{t_0\in \left[ n-\varDelta \right] ,\left\{ A_{\tau}\in \left[ k \right] \right\} _{\tau =t_0+1}^{t_0+\varDelta}}{\min}\,\, \lambda _{\min}\left( \sum_{\tau =t_0+1}^{t_0+\varDelta}{\boldsymbol{b}_{A_{\tau}}\left( \tau \right) \boldsymbol{b}_{A_{\tau}}\left( \tau \right) ^T} \right) .
$ 
Then we obtain that for $t-1 \leqslant \varDelta$, $
\frac{\lambda _{\min}\left( B\left( 1 \right) \right)}{4\left( t-1 \right)}\geqslant \frac{\lambda _{\min}\left( B\left( 1 \right) \right)}{4\varDelta}
$ and  for $t-1>\varDelta$,
$$
\begin{aligned}
    \lambda _{\min}\left( B\left( t \right) \right) & \geqslant \lambda _{\min}\left( B\left( 1 \right) \right) +\frac{t-1}{\varDelta}\rho _{\min}-\rho _{\min}
\\
&\geqslant \begin{cases}
	\frac{t-1}{\varDelta}\rho _{\min},   &if\,\,\lambda _{\min}\left( B\left( 1 \right) \right) \geqslant \rho _{\min},\\
	\frac{t-1}{\varDelta}\lambda _{\min}\left( B\left( 1 \right) \right) , &if\,\,\lambda _{\min}\left( B\left( 1 \right) \right) <\rho _{\min}.
\end{cases}
\end{aligned}
$$
Therefore, the parameter in Assumption \ref{con1} can be set as
$
\vartheta  =\frac{1}{4\varDelta}\min \left\{ \rho _{\min},\lambda _{\min}\left( B\left( 1 \right) \right) \right\} 
$.

\begin{lemma}\label{theo1}
    Let $\boldsymbol{\mu }$ be a instance generated as $\boldsymbol{\mu } \sim P_*=\mathcal{N} \left( \boldsymbol{\mu } _*,v^2\varSigma _* \right)$. Let $A_{t}^*  $ be the optimal arm in round $t$ under $\boldsymbol{\mu }$ and $A_t$ be the pulled arm in round $t$ by TS with prior $P_*$. $\vartheta  $ is the  constant    in Assumption  \ref{con1}.  Then for any $\delta>0$,
$
\mathbb{E} \left[ \sum_{t=1}^n{\left( \boldsymbol{b}_{A_{t}^*  }\left( t \right) -\boldsymbol{b}_{A_t}\left( t \right) \right) ^T\boldsymbol{\mu }} \right] \leqslant 2v\left( \sqrt{\frac{1}{\lambda _{\min}}}+\sqrt{\frac{n-1}{\vartheta }} \right) \sqrt{2\log \left( \frac{1}{\delta} \right)}    +nk\delta v\sqrt{\frac{2}{\pi \lambda _{\min}}}.
$
\end{lemma}



Lemma \ref{theo1} shows that using the true instance prior as the prior for TS leads to a reduction in regret, which is reasonable because the reduction in uncertainty about the bandit instance translates into lower regret.
Lemma \ref{theo2} establishes bounds on the difference in the $n$-round rewards achieved by TS  employing distinct priors.

\begin{lemma}\label{theo2}
    Let $\boldsymbol{\mu }$ be a   instance generated as $\boldsymbol{\mu } \sim P_*=\mathcal{N} \left( \boldsymbol{\mu } _*,v^2\varSigma _* \right)$.
  $\mathcal{N} \left( \hat{\boldsymbol{\mu}},v^2\varSigma _* \right) $ and $\mathcal{N} \left( \tilde{\boldsymbol{\mu}},v^2\varSigma _* \right) $ are two TS priors such that  $
        \left\| \hat{\boldsymbol{\mu}}-\tilde{\boldsymbol{\mu}} \right\| \leqslant \varepsilon
    $.  Let  $\hat{A}_t$ and $\tilde{A}_t$ be the pulled arms under these priors in round $t$. $\vartheta  $ is the  constant    in Assumption  \ref{con1}. Then for any $\delta>0$,
$    \mathbb{E} \left[ \sum_{t=1}^n{\left( \boldsymbol{b}_{\hat{A}_t}\left( t \right) -\boldsymbol{b}_{\tilde{A}_t}\left( t \right) \right) ^T\boldsymbol{\mu }} \right] \leqslant \frac{2k}{v}\left( \left\| \boldsymbol{\mu }_* \right\| +v\sqrt{\frac{2\log \left( \frac{1}{\delta} \right)}{\lambda _{\min}}} \right) \left( \sqrt{\frac{1}{\lambda _{\min}}}+\sqrt{\frac{n-1}{\vartheta }} \right) \sqrt{\frac{2\lambda _{\max}}{\pi}}\varepsilon +4nk\delta v\left( \sqrt{\frac{1}{2\pi \lambda _{\min}}}+\left\| \boldsymbol{\mu }_* \right\| \right) .$
\end{lemma}


Lemma \ref{theo2} states that the difference in the  rewards of TS when utilizing distinct priors can be quantitatively constrained by the difference in the prior means.

The pivotal dependency in Lemma \ref{theo2} lies in the linearity of the bound with respect to   $\varepsilon$.
Lemma 5 presented in \cite{metaTS} established an $O(n^2)$ upper bound on the discrepancy in the   rewards attained by TS employing distinct priors for Gaussian bandits. Indeed, under Assumption \ref{con1},  we have reduced the bound to $O(\sqrt{n})$ for linear bandits. 

\subsection{Regret Analysis on Meta-TSLB}
Lemma \ref{theo3G} demonstrates the concentration property of the sample means $\boldsymbol{\mu }_*$ derived from the meta-posterior distribution $Q_s$.
\begin{lemma} \label{theo3G}
    Let $\boldsymbol{\mu }_*\sim Q $
    and the meta-posterior in task $s$ is $Q_s$.
    If $\lambda _{\max}(\varSigma _Q )  \geqslant \frac{2}{175\lambda _{\min}}$, then
    $
        \left\| {\boldsymbol{\mu}}_{Q,s}-\boldsymbol{\mu }_* \right\| \leqslant v\sqrt{d\left( \left( \lambda _{\max}(\varSigma _Q )-\frac{2}{175\lambda _{\min}} \right) \left( \frac{7}{8} \right) ^{s-1}+\frac{2}{175\lambda _{\min}} \right) \log \left( \frac{2d}{\delta} \right)}
    $
    holds jointly over all tasks $s \in[m]$ with probability at least $1-m \delta$.
\end{lemma}

It can be found that when $s$ increases, the upper bound of  $\left\| {\boldsymbol{\mu}}_{Q,s}-\boldsymbol{\mu }_* \right\|$ gradually decreases.
That is, the meta-posterior concentrates as the number of tasks increases.
Now, we present the most significant conclusion of this work regarding the Bayes regret bound of Meta-TSLB.

\begin{theorem}\label{theo_mainG}
    $\vartheta  $ is the  constant    in Assumption  \ref{con1}.
    If $\lambda _{\max}(\varSigma _Q )  \geqslant \frac{2}{175\lambda _{\min}}$, 
    the Bayes regret of Meta-TSLB over $m$ tasks with
    $n$ rounds each is
$	R\left( m,n;P_* \right)  \leqslant 2mv\left( \sqrt{\frac{1}{\lambda _{\min}}}+\sqrt{\frac{n-1}{\vartheta }} \right) \sqrt{2\log \left( n \right)}+mkv\sqrt{\frac{2}{\pi \lambda _{\min}}}  
	   +\left\{2k\left( 4\log \left( m \right) u_2\left( \delta \right) +mu_3\left( \delta \right) \right) \left( u_1\left( \delta \right) +v\sqrt{\frac{2\log \left( n \right)}{\lambda _{\min}}} \right)\right.\\ \left.\cdot \left( \sqrt{\frac{1}{\lambda _{\min}}}+\sqrt{\frac{n-1}{\vartheta }} \right) \sqrt{\frac{2\lambda _{\max}}{\pi}}\right\} 
	  +4mkv\left( \sqrt{\frac{1}{2\pi \lambda _{\min}}}+u_1\left( \delta \right) \right) $,
    with probability at least $1- ( m+ 1) \delta$, where
    $
        u_1\left( \delta \right) =\left\| \boldsymbol{\mu }_Q \right\| +v\sqrt{d\lambda _{\max}(\varSigma _Q )\log \left( \frac{2d}{\delta} \right)} $, $  u_2\left( \delta \right) =\sqrt{d\left( \lambda _{\max}(\varSigma _Q )-\frac{2}{175\lambda _{\min}} \right) \log \left( \frac{2d}{\delta} \right)}$ and  $u_3\left( \delta \right) =\sqrt{d\frac{2}{175\lambda _{\min}}\log \left( \frac{2d}{\delta} \right)}.
    $
    The probability is over realizations of $
        \boldsymbol{\mu }_*,\boldsymbol{\mu }_s$.

\end{theorem}

If we only focus on the number of tasks $m$ and the number of rounds $n$, our bound  can be summarized as $
O\left( \left( m+\log \left( m \right) \right) \sqrt{n\log \left( n \right)} \right)$, which is an improvement compared to the $
O\left( \,\,m\sqrt{n\log \left( n \right)}+\sqrt{m}\left( n^2\sqrt{\log \left( n \right)} \right)\right)$ bound  of Meta-TS given by Kveton et al.   for the Gaussian bandits, which are a special case of linear bandits.
Our bound presented in Theorem \ref{theo_mainG} can be interpreted as follows: The initial two terms represent the regret incurred by TS when equipped with the accurate prior $P_*$. This regret scales linearly with the number of tasks $m$, reflecting the fact that Meta-TSLB tackles $ m$ distinct exploration problems. The subsequent two terms encapsulates the expense of learning $P_*$, which is sublinear with respect to $m$. Consequently, in scenarios where $m$ is large, Meta-TSLB demonstrates near-optimal performance, underscoring its efficiency and effectiveness.


\subsection{Regret Analysis on Meta-TS applied to Linear Bandits}
The analysis of the Bayes regret bound for Meta-TS given in \cite{metaTS} is limited to Gaussian bandits.
In this subsection, by leveraging the proof of Theorem \ref{theo_mainG}, we can directly derive the bound for Meta-TS   when applied to linear bandits.
\begin{lemma} \label{theo3}
    Let $\boldsymbol{\mu }_*\sim Q $
    and the prior parameters in task $s$  of Meta-TS be sampled as $
        \hat{\boldsymbol{\mu}}_s|H_{1:s}\sim Q_s
    $.  If $\lambda _{\max}(\varSigma _Q )  \geqslant \frac{2}{175\lambda _{\min}}$, then
    $
        \left\| {\boldsymbol{\mu}}_{Q,s}-\boldsymbol{\mu }_* \right\| \leqslant 2v\sqrt{d\left( \left(  \lambda _{\max}\left( \varSigma _{Q} \right)-\frac{2}{175\lambda _{\min}} \right) \left( \frac{7}{8} \right) ^{s-1}+\frac{2}{175\lambda _{\min}} \right) \log \left( \frac{4d}{\delta} \right)}
    $
    holds jointly over all tasks $s \in[m]$ with probability at least $1-m \delta$.
\end{lemma}



\begin{theorem}\label{theo_main}
       $\vartheta  $ is the  constant    in Assumption  \ref{con1}.  If $\lambda _{\max}(\varSigma _Q )  \geqslant \frac{2}{175\lambda _{\min}}$,
    the Bayes regret of Meta-TS over $m$ tasks with
    $n$ rounds each is
    $ 	R\left( m,n;P_* \right) \leqslant 2mv\left( \sqrt{\frac{1}{\lambda _{\min}}}+\sqrt{\frac{n-1}{\vartheta }} \right) \sqrt{2\log \left( n \right)}+mkv\sqrt{\frac{2}{\pi \lambda _{\min}}}
	 +\left\{4k\left( 4\log \left( m \right) u_4\left( \delta \right) +mu_5\left( \delta \right) \right) \left( u_1\left( \delta \right) +v\sqrt{\frac{2\log \left( n \right)}{\lambda _{\min}}} \right)\right. \\ \left.\cdot \left( \sqrt{\frac{1}{\lambda _{\min}}}+\sqrt{\frac{n-1}{\vartheta }} \right) \sqrt{\frac{2\lambda _{\max}}{\pi}}\right\}
	+4mkv\left( \sqrt{\frac{1}{2\pi \lambda _{\min}}}+u_1\left( \delta \right) \right) $,
    with probability at least $1- ( m+ 1) \delta$, where
    $
        u_4\left( \delta \right) =\sqrt{d\left(  \lambda _{\max}\left( \varSigma _{Q} \right)-\frac{2}{175\lambda _{\min}} \right) \log \left( \frac{4d}{\delta} \right)},\ u_5\left( \delta \right) =\sqrt{d\frac{2}{175\lambda _{\min}}\log \left( \frac{4d}{\delta} \right)}
    $ 
    and $u_1\left( \delta \right)$ is defined in Theorem \ref{theo_mainG}.
    The probability is over realizations of $
        \boldsymbol{\mu }_*,\boldsymbol{\mu }_s$ and $\hat{\boldsymbol{\mu}}_s.
    $

\end{theorem}



Comparing Theorem \ref{theo_mainG} with Theorem \ref{theo_main}, it can be found that the Bayes regret bound of Meta-TSLB is smaller than that of Meta-TS. This can be attributed to the fact that the bound of  $\left\| {\boldsymbol{\mu}}_{Q,s}-\boldsymbol{\mu }_* \right\|$ is smaller than that of $\left\| \hat{\boldsymbol{\mu}}_s-\boldsymbol{\mu }_* \right\|$.
    Specifically,  for task $s$, the prior  $\mathcal{N}(\boldsymbol{\mu }_{Q,s},v^2 \varSigma _* )$ of Meta-TSLB is more likely to be closer to $P_*=\mathcal{N}(\boldsymbol{\mu }_*,v^2 \varSigma _* )$ than prior $\mathcal{N}( \hat{\boldsymbol{\mu}}_s,v^2 \varSigma _* )$ of Meta-TS.

\section{Extended version of Linear Bandits}
\subsection{Linear Bandits with  Finite Potential Instance Priors}\label{SecLBF}

In this subsection,  we  assume access to $L$ potential instance priors $\mathcal{P} = \left \{ P^{( j) }\right \} _{j= 1}^L$, where $
    P^{(j)}=\mathcal{N} \left( \boldsymbol{\mu }^{\left( j \right)},v^2\varSigma ^{\left( j \right)} \right) $ for some fixed $\boldsymbol{\mu }^{\left( j \right)} $ and $\varSigma ^{\left( j \right)}$, $j=1,\cdots,L$.
The meta-prior $Q$ is the probability mass function on $L$  potential instance priors, that is, $Q(j)= w_j$  is the probability that $P^{(j)} $ is choosed and $\sum^L_{j=1}w_j=1$.
The tasks are generated as follows. First, the instance prior is set as $P_*=P^{(j_*)}$ where $j_*\sim Q$. Then, in each task $s$, a   instance is sampled as $\boldsymbol{\mu }_s \sim P_*$.

Meta-TSLB is implemented as follows. The meta-posterior in
task $s$ is $Q_s(j)={w}_{s,j} $,
where $ \boldsymbol{w}_s=(w_{s,1},\cdots,w_{s,L})$ is a vector of posterior beliefs into each instance prior. The instance prior in task $s$ is $P_s=P^{(j_s)}$  such that $w_{s,j_s}$ is the biggest element of $\boldsymbol{w}_s$. Suppose that  in task $s$, the pulled arm in round $t$ of TS is $A_t$.  After interacting with bandit instance $\boldsymbol{\mu }_s $, the meta-posterior is updated using $Q_{s+1}(j)\propto f(j)Q_s(j)$, where
$
    f\left( j \right) =\exp \left\{ -\frac{1}{2v^2}\left[ \left[ \boldsymbol{\mu }^{\left( j \right)} \right] ^T\left[ \varSigma ^{\left( j \right)} \right] ^{-1}\boldsymbol{\mu }^{\left( j \right)}-\boldsymbol{\xi }^TG\boldsymbol{\xi } \right] \right\}
$
and $  \boldsymbol{\xi }=G^{-1}\left( \sum_{t=1}^n{r_{A_t}\left( t \right) \boldsymbol{b}_{A_t}\left( t \right) ^T}+\left[ \boldsymbol{\mu }^{\left( j \right)} \right] ^T\left[ \varSigma ^{\left( j \right)} \right] ^{-1} \right) ^T$, 
$G=\sum_{t=1}^n{\boldsymbol{b}_{A_t}\left( t \right) \boldsymbol{b}_{A_t}\left( t \right) ^T}+\left[ \varSigma ^{\left( j \right)} \right] ^{-1}$. This conclusion can be directly derived from the proof of Lemma \ref{Qs}.

\subsection{Linear Bandits with Infinite Arms}\label{SecLBI}

A linear bandit instance with infinite arms is  parameterized by a vector $\boldsymbol{\mu } \in \mathbb{R}^d$. In the context of a $k$-arm linear bandit, the agent observes $k$ context vectors in round $t$, each uniquely associated with an arm. Conversely, in the realm of linear bandits with infinite arms, the agent observes a polyhedron $\mathcal{B}(t)$ in round $t$ and needs to select a context vector $\boldsymbol{b}(t)\in \mathbb{R}^d$ from   $\mathcal{B}(t)$ as the pulled arm. The reward $r\left( t \right)$ for this selected arm is then drawn  i.i.d.  from   $\mathcal{N}( \boldsymbol{b } \left( t \right) ^T\boldsymbol{\mu },v^2)$.

Thus, in TS (Algorithm \ref{alg2}), the step 5 needs to be modified to \enquote{
 Select $ \boldsymbol{b}^{\prime}\left( t \right) =\mathrm{arg}\underset{\boldsymbol{b}\left( t \right) \in \mathcal{B} \left( t \right)}{\max}\,\,\boldsymbol{b}\left( t \right) ^T\tilde{\boldsymbol{\mu}}\left( t \right) $ and observe $r\left( t \right) \sim \mathcal{N} \left( \boldsymbol{b }^{\prime}\left( t \right) ^T\boldsymbol{\mu } ,v^2 \right) $ }.

\subsection{Sequential Linear  Bandits}\label{SecSLB}
In this subsection, we define a new problem   named sequential linear bandit. A sequential linear bandit instance is   parameterized by $\boldsymbol{\mu}\in \mathbb{R}^d$, and contains $p$ linear bandits (Bandit 1, Bandit 2, $\cdots$, Bandit $p$). Figure \ref{P_SLB} shows the case of $p=3$. 
In round $t$, the agent needs to select one arm from   Bandit 1 to Bandit $p$ in sequence. Let $A^{(t)}_i$ denote the arm pulled in Bandit $i\in [p]$ and $\varGamma $ be a mapping,
\begin{equation}\label{eq12}
    \boldsymbol{b}\left( t \right) =\varGamma \left( \boldsymbol{b}_{1,A_{1}^{\left( t \right)}}\left( t \right) ,\boldsymbol{b}_{2,A_{2}^{\left( t \right)}}\left( t \right) ,\cdots ,\boldsymbol{b}_{p,A_{p}^{\left( t \right)}}\left( t \right) \right) \in \mathbb{R} ^d.
\end{equation}
The image of the context vectors corresponding to $p$ pulled arms under   mapping $\varGamma$  is a $d$-dimensional vector. Denote
$
\psi \left( \boldsymbol{\mu },A_{1}^{\left( t \right)},A_{2}^{\left( t \right)},\cdots ,A_{p}^{\left( t \right)} \right) =\boldsymbol{b}\left( t \right) ^T\boldsymbol{\mu }.
$
Then agent receives a reward  $r(t)$, which is drawn i.i.d. from $
    \mathcal{N} \left( \boldsymbol{b}\left( t \right) ^T\boldsymbol{\mu },v^2 \right).
$ Obviously, when $p = 1$ and $
    \varGamma $ is
identity mapping, sequential linear bandits is linear bandits  discussed in Section 2.
\begin{figure}[t]
    \centering
    \includegraphics[width=0.8\linewidth]{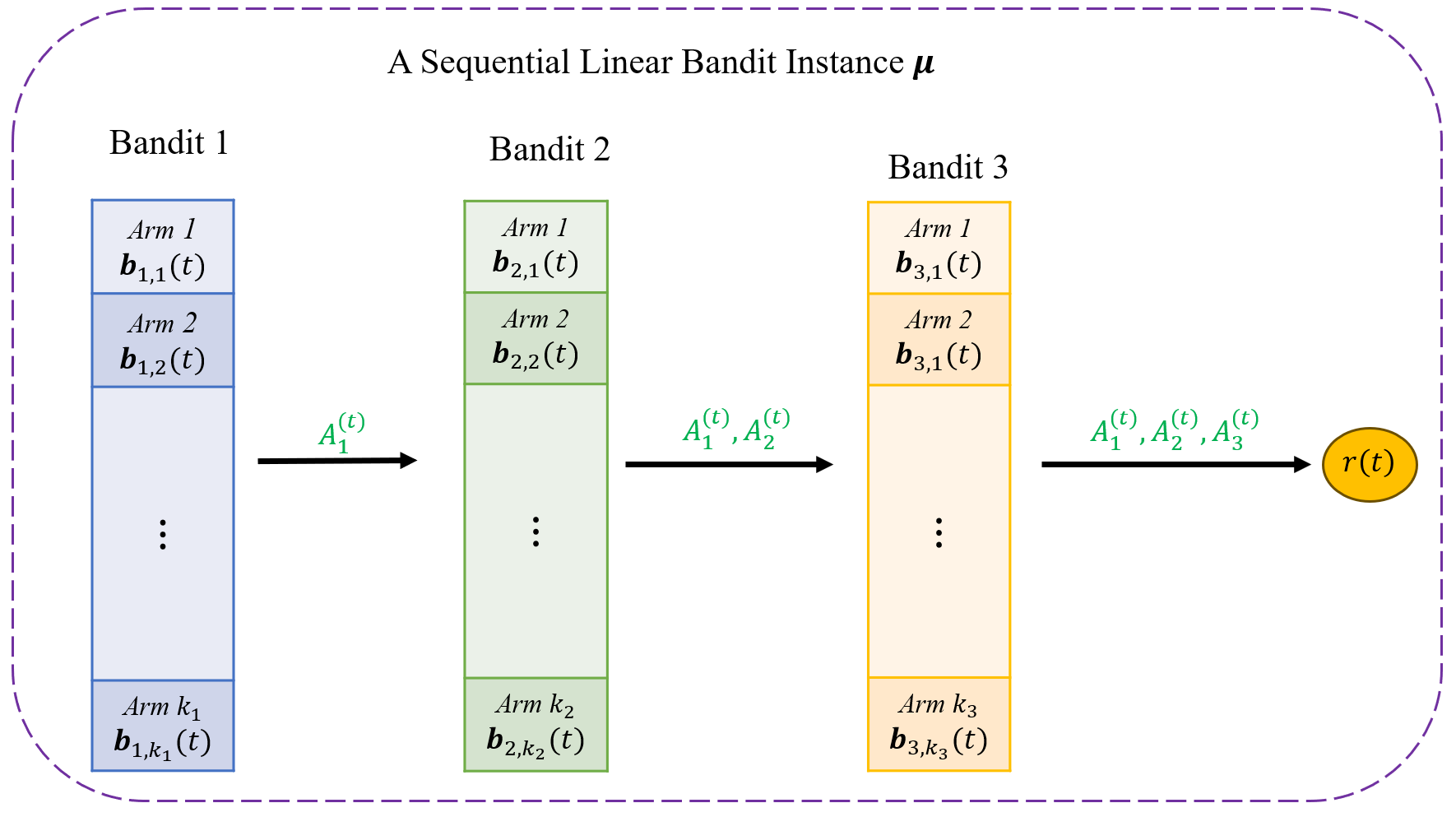}
    \caption{A sequential linear  bandit instance with $p=3$}
    \label{P_SLB}
\end{figure}


Now, we employ Thompson sampling to this specific instance. Note that in round t, when the agent needs to  pull an arm in Bandit $i$, it remains oblivious to the arm pulled in Bandit $j$ for $j>i$. Consequently, the optimal arm to pull in Bandit $i$ cannot be definitively determined. To circumvent this issue, we leverage the context vector of the arm pulled in round $t-1$ of Bandit $j$ as a predictive context vector for round $t$. The detailed procedural steps are outlined in Algorithm \ref{alg3}. By substituting  TS  with Algorithm \ref{alg3} within Meta-TSLB, we derive Meta-TSLB variants tailored specifically to address this problem.

\begin{algorithm}[h]
    {\bf {Setting:}}  Bandit instance $\boldsymbol{\mu }$; 

    {\bf {Input:}}  Prior  $P(1)=\mathcal{N} \left( \hat{\boldsymbol{\mu }}\left( 1 \right) ,v^2B\left( 1 \right) ^{-1} \right)$

    {\bf {For}} $t=1,2,\cdots,n$ {\bf {do}}

    \qquad Sample  $\tilde{\boldsymbol{\mu }}(t) \sim  P(t)= \mathcal{N} \left( \hat{\boldsymbol{\mu }}\left( t \right) ,v^2B\left( t \right) ^{-1} \right)$

    \qquad {\bf {For}} $i=1,2,\cdots,p$ {\bf {do}}

    \qquad\qquad Pull arm $
A_{i}^{\left( t \right)}=\mathrm{arg}\underset{A}{\max}\,\,\psi \left( \tilde{\boldsymbol{\mu}}(t),A_{1}^{\left( t \right)},\cdots ,A_{i-1}^{\left( t \right)},A,A_{i+1}^{\left( t-1 \right)}\cdots ,A_{p}^{\left( t-1 \right)} \right) 
$

    \qquad {\bf {End For}}

    \qquad  Observe reward $r \left( t \right) \sim \mathcal{N} \left( \boldsymbol{b }\left( t \right) ^T\boldsymbol{\mu } ,v^2 \right)
    $, where $\boldsymbol{b }\left( t \right)$ is defined in \eqref{eq12},  and update $P(t)$ to $P(t+1)$

    {\bf {End For}}
    \caption{Thompson Sampling  for Sequential Linear Bandits.}
    \label{alg3}
\end{algorithm}
\section{Experiment}

In this section, we  demonstrate the efficacy of Meta-TSLB through a series of   experiments. Each experiment comprises $m=20$ tasks, spanning a horizon of $n=200$ rounds, and all outcomes are averaged across 100 independent runs to ensure robustness, where $P^* \sim Q$ in each run.
We maintain a consistent setup with a reward standard deviation of $v=0.2$ and a context vector dimensionality of $d=5$. Except for the third and fourth experiments, the number of available arms is fixed at $k=20$. The mean vector, denoted as $\boldsymbol{\mu}_Q$ , is initialized to the zero vector $\boldsymbol{0}_d$, while the covariance matrices  $\varSigma _Q$, $\varSigma _*$ are randomly generated as symmetric, non-diagonal matrices, constrained to have element values less than 3.
The context vectors are sampled uniformly at random from the interval $[0, 50]^d$. To assess the performance of Meta-TSLB, we benchmark it against Meta-TS and two variations of Thompson Sampling: OracleTS, which assumes knowledge of the instance prior $P^*$, and TS, which marginalizes out the meta-prior $Q$.

The first experiment is centered on normal linear bandits, with its setting grounded in Section 2. The outcomes   are   represented in Figure \ref{LB}.

The second experiment is set up as described in Subsection \ref{SecLBF}, focused on the linear bandits with finite potential instance priors. Here, we randomly generate $L=50$ distributions  $\left \{ P^{( j) }\right \} _{j= 1}^{50}$ as potential instance priors, where the mean  $\boldsymbol{\mu}^{(j)}$ of   distributions $P^{( j) }$ is obtained by randomly sampling from $[-1,1]^d$, and $\varSigma^{(j)}$ is  a randomly generated matrix with element values less than 3, $j=1,\cdots,50 $. The results of the second experiment   are illustrated in Figure \ref{LBF}.

The third experiment is about the  linear bandits with infinite arms, whose settings can be found in Subsection \ref{SecLBI}.
For the  polyhedron $\mathcal{B}(t)$ int round $t$, we randomly generate a matrix $A(t)\in\mathbb{R}^{5\times d}$ and a vector $B(t)\in\mathbb{R}^{5}$, subsequently set $\mathcal{B} (t)=\left\{ \boldsymbol{x}\in \mathbb{R} ^d|A\left( t \right) \boldsymbol{x}\leqslant B\left( t \right) \right\} $. The results are shown in Figure \ref{LBI}.

The fourth experiment delves into sequential linear bandits, with its setting outlined in Section \ref{SecSLB}. In this experiment, we postulate that each sequential bandit instance comprises $p=3$ distinct linear bandits: Bandit 1, Bandit 2, and Bandit 3.
Let
$
    \boldsymbol{b}\left( t \right) =\boldsymbol{b}_{1,A_{1}^{\left( t \right)}}\left( t \right) \circ \boldsymbol{b}_{2,A_{2}^{\left( t \right)}}\left( t \right) \circ \boldsymbol{b}_{p,A_{3}^{\left( t \right)}}\left( t \right) ,
$
where the symbol $\circ$ represents the Hadamard product (element-wise product).
We assign the number of arms for Bandits 1, 2, and 3 to be 20, 15, and 5, respectively. Presuppose that the initial arm pulled in round 0 is identical for all three bandits, specifically $A_1^{(0)}=A_2^{(0)}=A_3^{(0)}=1$. The outcomes of this experiment are illustrated in Figure \ref{TSSLB}.

Evidently, OracleTS attains the lowest possible regret. It can be observed that Meta-TSLB outperforms Meta-TS, because for task $s$, the prior  of Meta-TSLB is more likely closer to  $P^*$ compared to the prior  of Meta-TS.





\begin{figure}[t] 
    \centering 

    \begin{minipage}[t]{0.48\linewidth} 
        \centering
        \includegraphics[width=\linewidth]{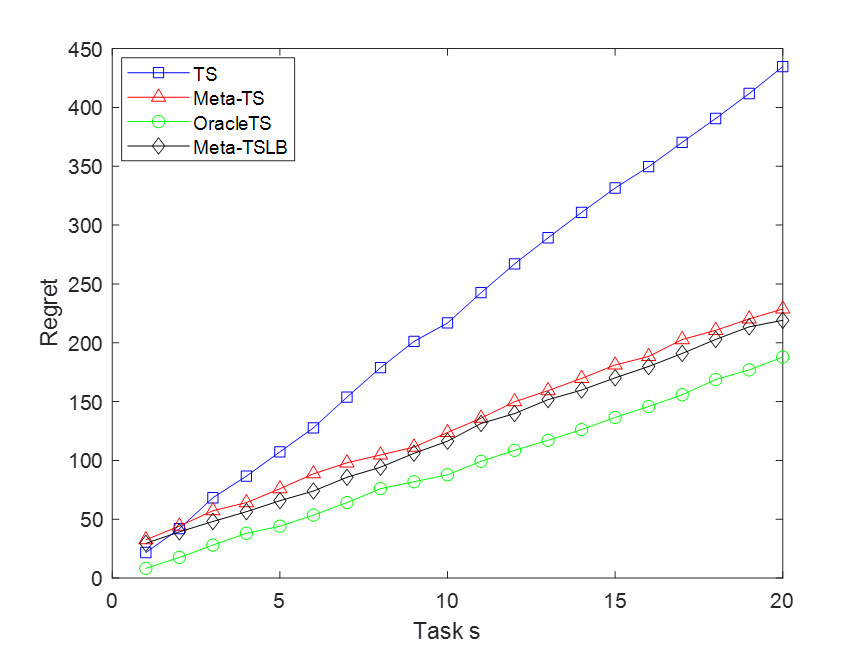} 
        \subcaption{ Linear bandits}
        \label{LB}
    \end{minipage}%
    \hfill 
    \begin{minipage}[t]{0.48\linewidth}
        \centering
        \includegraphics[width=\linewidth]{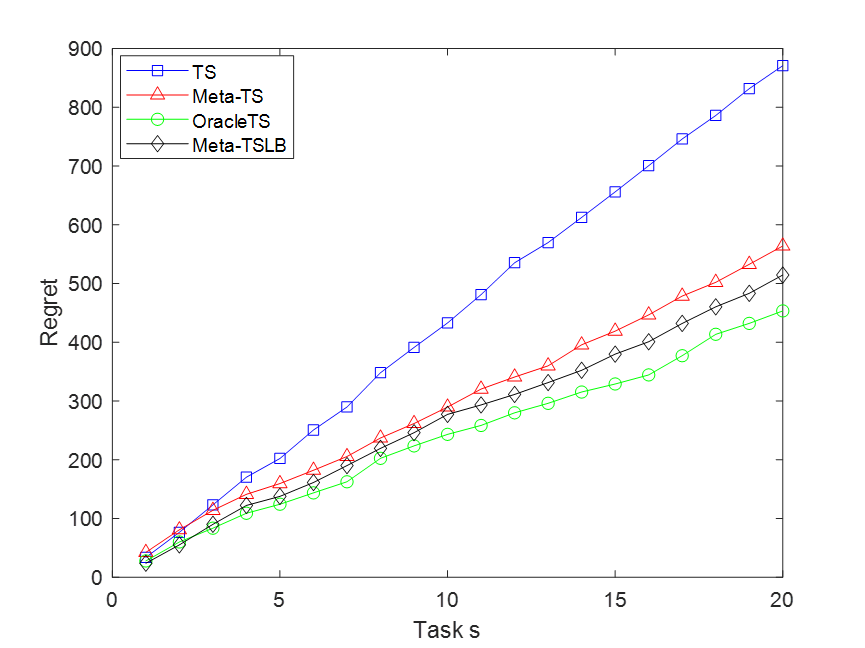} 
        \subcaption{Linear bandits with finite potential instance priors}
        \label{LBF}
    \end{minipage}

    \vskip\baselineskip

    \begin{minipage}[t]{0.48\linewidth}
        \centering
        \includegraphics[width=\linewidth]{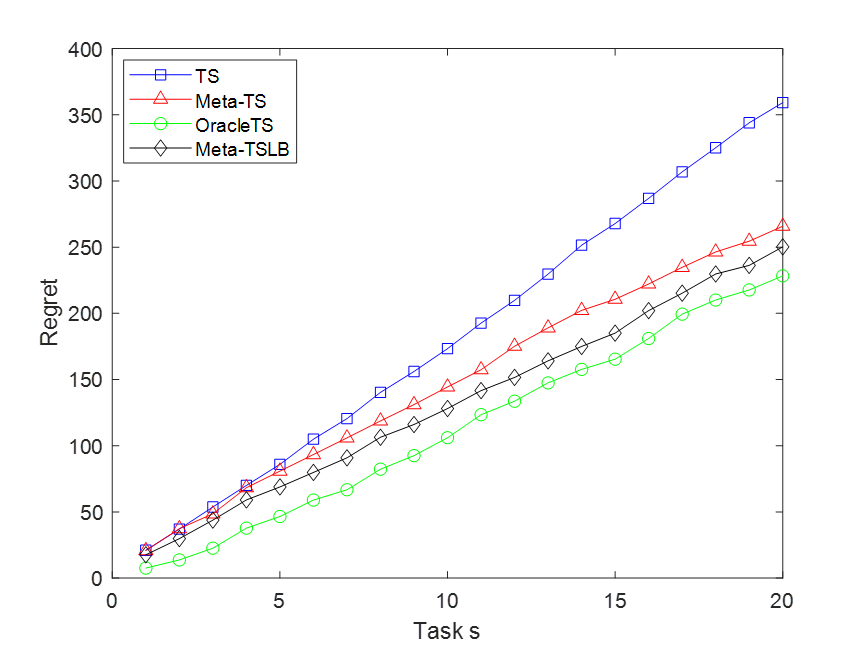} 
        \subcaption{Linear bandits with infinite arms}
        \label{LBI}
    \end{minipage}%
    \hfill
    \begin{minipage}[t]{0.48\linewidth}
        \centering
        \includegraphics[width=\linewidth]{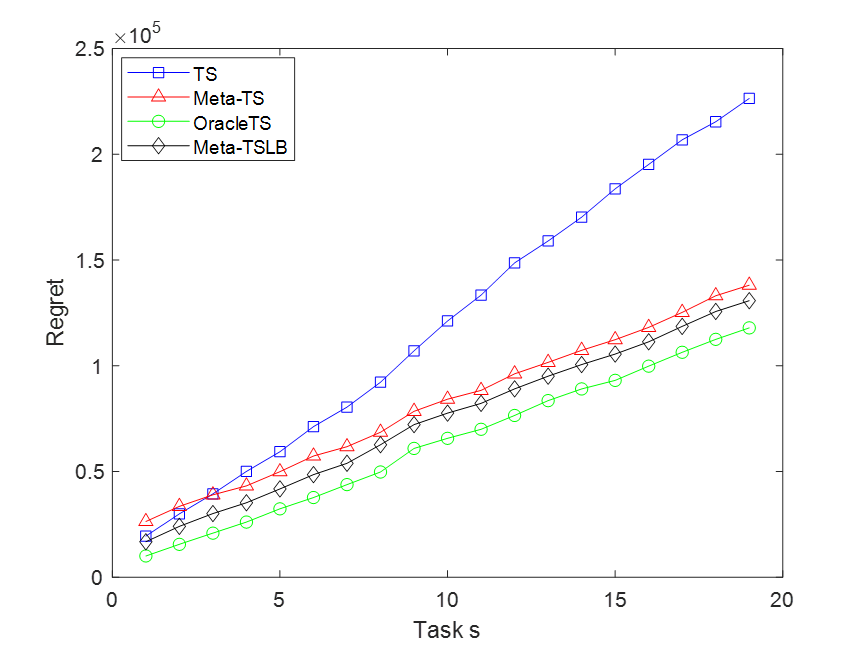} 
        \subcaption{Sequential linear bandits}
        \label{TSSLB}
    \end{minipage}

    \caption{Comparison of Meta-TSLB with Meta-TS,  OracleTS and TS in different settings}
    \label{fig2}
\end{figure}


The final experiment aims to validate the generalization ability of Meta-TSLB. To facilitate this assessment, we employ the ordinary linear bandits detailed in Section 2 for testing.  $\mathcal{M} \left( \boldsymbol{\mu }_* \right)$ is the set of $m$ tasks, sampled from the instance prior $P_*=\mathcal{N}(\boldsymbol{\mu }_*,v^2 \varSigma _* )$ .
By applying Meta-TSLB (Meta-TS) to $\mathcal{M} \left( \boldsymbol{\mu }_* \right)$, we derive the meta-posterior $Q^{\prime}$. This meta-posterior is then leveraged as the meta-prior for a fresh set of tasks $    \mathcal{M} \left( \boldsymbol{\mu }_*+\boldsymbol{\epsilon } \right)$, where $\boldsymbol{\epsilon }$ is a randomly generated vector. This approach allows us to evaluate the generalization  ability of our algorithms under slight variations in the task distribution.


The results of experiments conducted with varying $\left\| \boldsymbol{\epsilon } \right\|$ values of 0, 1, 3, and 6 are presented in Figure \ref{fig3}. A notable observation is that Meta-TSLB and Meta-TS exhibit comparable generalization abilities. Specifically, when $\left\| \boldsymbol{\epsilon } \right\|$ is small, both algorithms display remarkable generalization abilities, with their performance rivaling that of OracleTS at $\left\| \boldsymbol{\epsilon } \right\|=0$. This indicates that the meta-posterior $Q^{\prime}$ after $m$ iterations is very close to the instance prior $P^*$.
Furthermore, as $\left\| \boldsymbol{\epsilon } \right\|$ increases, both Meta-TSLB and Meta-TS maintain good performance for subsequent tasks after learning from multiple tasks, meaning that their meta-posteriors approaching $ \mathcal{N}(\boldsymbol{\mu }_*+\boldsymbol{\epsilon },v^2 \varSigma _* )$ after learning from multiple tasks.
Denote the Bayes regret bound with meta-prior $Q$  and meta-prior $Q^{\prime}$ by $R_{Q}$ and $R_{Q^{\prime}}$. According to Lemma \ref{theo3G},  for the task $s$ in  $    \mathcal{M} \left( \boldsymbol{\mu }_*+\boldsymbol{\epsilon } \right)$, we have
$
    \left\| \boldsymbol{\mu }_{Q,s}-\left( \boldsymbol{\mu }_*+\boldsymbol{\epsilon } \right) \right\|  \leqslant \left\| \boldsymbol{\epsilon } \right\| +v\sqrt{d\left( \left( \lambda _{\max}(\varSigma _Q ) -\frac{2}{175\lambda _{\min}} \right) \left( \frac{7}{8} \right) ^{m-1}+\frac{2}{175\lambda _{\min}} \right) \log \left( \frac{2d}{\delta} \right)}.
$
By the proof of Theorem \ref{theo_mainG}, we get
$	R_Q-R_{Q^{\prime}}\propto  v\left( 4\log \left( m \right) -m\left( \frac{7}{8} \right) ^{\frac{m}{2}} \right) u_2\left( \delta \right)-m\left\| \boldsymbol{\epsilon } \right\| .$
Since  $4\log \left( m \right) -m\left( \frac{7}{8} \right) ^{\frac{m}{2}} >0$ for $m\geqslant 2$, if
$
    \left\| \boldsymbol{\epsilon } \right\| \leqslant v\left( 4\frac{\log \left( m \right)}{m}-\left( \frac{7}{8} \right) ^{\frac{m}{2}} \right) u_2\left( \delta \right),$
 we conclude that $R_Q>R_{Q^{\prime}}$.


\begin{figure}[t] 
    \centering 

    \begin{minipage}[t]{0.48\linewidth} 
        \centering
        \includegraphics[width=\linewidth]{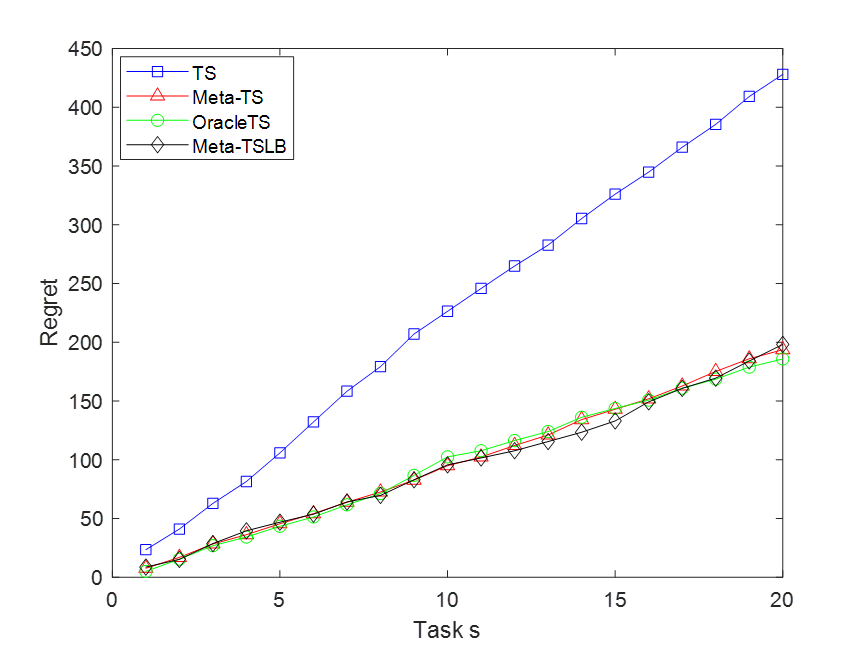} 
        \subcaption{$\left\| \boldsymbol{\epsilon }  \right\| =0$}
        \label{ eee}
    \end{minipage}%
    \hfill 
    \begin{minipage}[t]{0.48\linewidth}
        \centering
        \includegraphics[width=\linewidth]{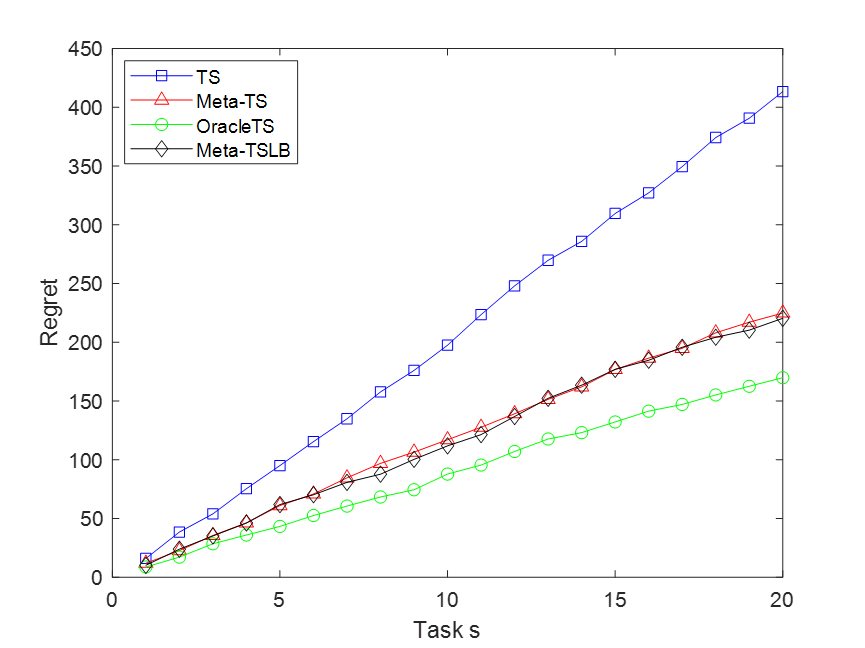} 
        \subcaption{$\left\| \boldsymbol{\epsilon } \right\| =1$}
        \label{fig:top-right}
    \end{minipage}

    \vskip\baselineskip

    \begin{minipage}[t]{0.48\linewidth}
        \centering
        \includegraphics[width=\linewidth]{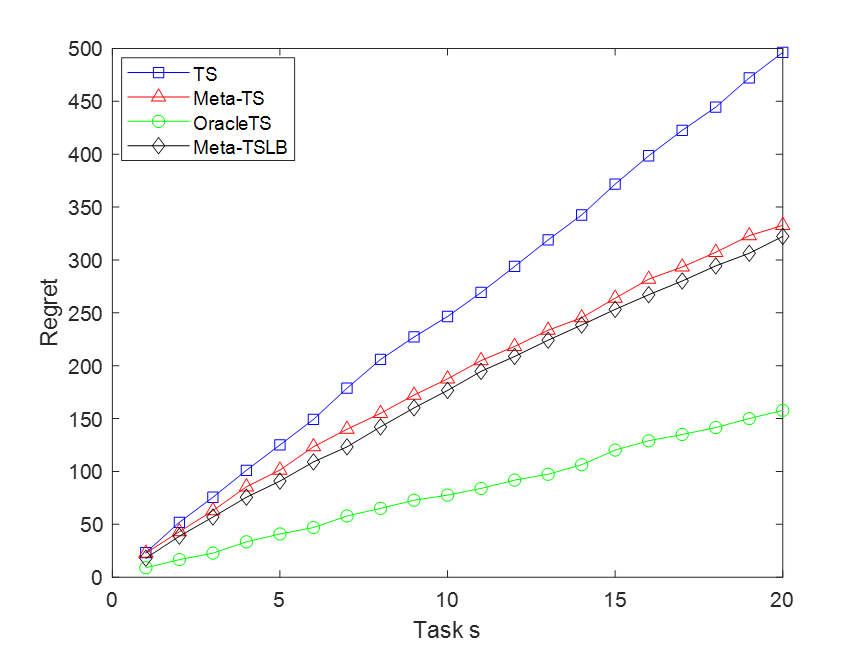} 
        \subcaption{$\left\| \boldsymbol{\epsilon } \right\| =3$}
        \label{fig:bottom-left}
    \end{minipage}%
    \hfill
    \begin{minipage}[t]{0.48\linewidth}
        \centering
        \includegraphics[width=\linewidth]{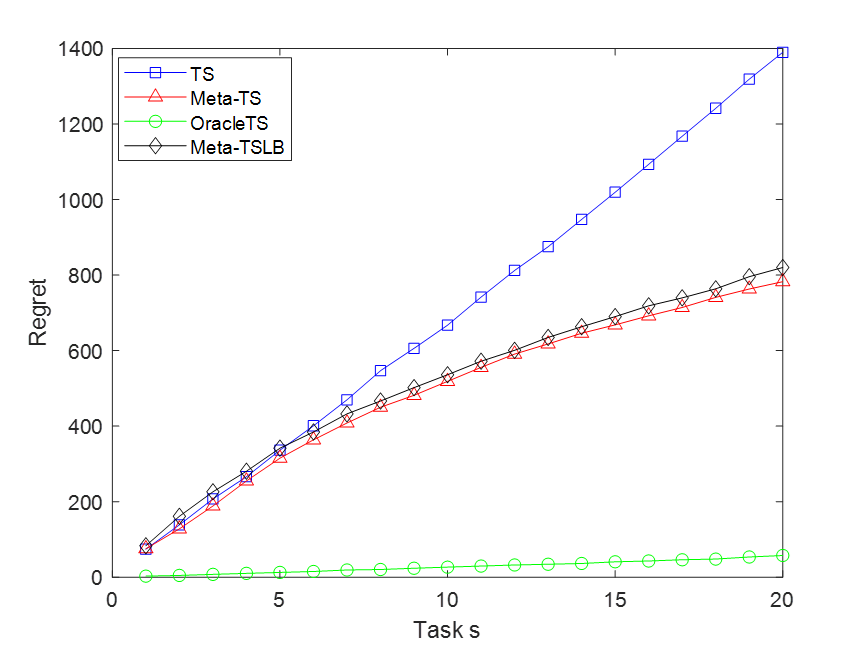} 
        \subcaption{$\left\| \boldsymbol{\epsilon } \right\| =6$}
        \label{fig:bottom-right}
    \end{minipage}

    \caption{ The generalization ability testing experiment of Meta-TSLB and  Meta-TS  }
    \label{fig3}
\end{figure}

\section{Conclusion}
This paper explores the extension of TS for linear bandits under a meta-learning framework. We introduce Meta-TSLB algorithm, which leverages a meta-prior and meta-posterior distributions to model the  uncertainty in the instance prior, allowing the learning agent to adaptively update its prior based on sequential interactions with bandit instances. 
Theoretical analyses provide an $ O\left( \left( m+\log \left( m \right) \right) \sqrt{n\log \left( n \right)} \right)$  bound on  Bayes regret, indicating that as the agent learns about the unknown prior, its performance improves.
We also complemente  the Bayes regret bound of Meta-TS applied to linear bandits, and the bound of Meta-TSLB is smaller because for all tasks, the prior of Meta-TSLB is closer to the instance prior compared to that of Meta-TS.
Extensive experiments on various linear bandit settings, including finite potential instance priors, infinite arms and sequential linear bandits, demonstrate the effectiveness of Meta-TSLB, showing that it outperforms Meta-TS \cite{metaTS} and TS  with incorrect priors and approaches the performance of TS with known priors (OracleTS).
Furthermore, we   demonstrate and analyze the    generalization  ability of Meta-TSLB, highlighting its potential to adeptly adapt to novel and unseen linear bandit tasks.

\newlength{\originaltextheight}  
\setlength{\originaltextheight}{\textheight}

\addtolength{\textheight}{-16cm}   

\bibliographystyle{IEEEtran}

\bibliography{IEEEabrv, ref}
\setlength{\textheight}{\originaltextheight}  
\onecolumn
\section*{APPENDIX}

\subsection{Preliminary Lemmas}
We first define some  symbols for use in all subsequent proofs. Let    $ \hat{\theta}_i\left( t \right) =\boldsymbol{b }_i\left( t \right) ^T\hat{\boldsymbol{\mu }}\left( t \right) $, $\tilde{\theta}_i\left( t \right) =\boldsymbol{b }_i\left( t \right) ^T\tilde{\boldsymbol{\mu }}\left( t \right) $  and $s_{i}\left( t \right)^{2} =\boldsymbol{b }_i\left( t \right) ^TB\left( t \right) ^{-1}\boldsymbol{b }_i\left( t \right) $, where $\hat{\boldsymbol{\mu }}(t)$, $\tilde{\boldsymbol{\mu }}(t)$ and $B\left( t \right)$ are defined in TS. Then, we obtain 
$\tilde{\theta}_i\left( t \right) \sim \mathcal{N} \left( \hat{\theta}_i\left( t \right) ,v^2s_{i}\left( t \right)^{2} \right)$,
that is, $ vs_{i}\left( t \right)$  the standard deviation of random variable $\tilde{\theta}_i\left( t \right) $.

\begin{lemma}\label{lem10}
Suppose that $\vartheta $ is the constant    in Assumption  \ref{con1}, then $
\sum_{t=1}^n{\sqrt{\frac{1}{\lambda _{\min}\left( B\left( t \right) \right)}}}<\sqrt{\frac{1}{\lambda _{\min}\left( B\left( 1 \right) \right)}}+\sqrt{\frac{n-1}{\vartheta}}.
$
\end{lemma}
\begin{proof}
    For $t\geqslant 2$, we have $
\sqrt{\frac{1}{\lambda _{\min}\left( B\left( t \right) \right)}}\leqslant \frac{1}{2\sqrt{\vartheta \left( t-1 \right)}}< \frac{1}{\sqrt{\vartheta}}\left( \frac{1}{\sqrt{t-1}+\sqrt{t-2}} \right) =\frac{1}{\sqrt{\vartheta}}\left( \sqrt{t-1}-\sqrt{t-2} \right) .
$
    Thus, 
    $
\sum_{t=1}^n{\sqrt{\frac{1}{\lambda _{\min}\left( B\left( t \right) \right)}}}<\sqrt{\frac{1}{\lambda _{\min}\left( B\left( 1 \right) \right)}}+\frac{1}{\sqrt{\vartheta}}\sum_{t=2}^n{\left( \sqrt{t-1}-\sqrt{t-2} \right)}=\sqrt{\frac{1}{\lambda _{\min}\left( B\left( 1 \right) \right)}}+\sqrt{\frac{n-1}{\vartheta}}.
$
\end{proof}

 According to  Rayleigh theorem and \eqref{eq18},  we obatin that 
\begin{equation}\label{eq17}
        s_i\left( t \right)^2 \leqslant {\frac{1}{\lambda _{\min}\left( B\left( t \right) \right)}}\left\| \boldsymbol{b}_i\left( t \right) \right\|^2 \leqslant {\frac{1}{\lambda _{\min}\left( B\left( t \right) \right)}}\leqslant {\frac{1}{\lambda _{\min}  (B(1)) }},
\end{equation}
 and under Assumption \ref{con1},
\begin{equation}\label{eq16}
\sum_{t=1}^n{s_i\left( t \right)}\leqslant \sum_{t=1}^n{\sqrt{\frac{1}{\lambda _{\min}\left( B\left( t \right) \right)}}}<{\frac{1}{\lambda _{\min}  (B(1)) }}+\sqrt{\frac{n-1}{\vartheta}}.
\end{equation}

\subsection{Proof of Lemma \ref{Qs}}
\begin{proof}
 Let $\varPhi \left( \boldsymbol{x};\boldsymbol{\theta },\varSigma \right)$ be the probability density function of multivariate Gaussian distribution $\mathcal{N} \left( \boldsymbol{\theta },\varSigma \right)$ and $\varphi \left( x;\theta ,v ^2 \right) $  be the probability density function of Gaussian distribution   $\mathcal{N} \left( \theta ,v^2 \right) $.
    According to \cite{metaTS}, once the task $s$ is complete, it updates the meta-posterior in a standard Bayesian fashion
    $
        Q_{s+1}\left( \bar{\boldsymbol{\mu }} \right)
        \propto f\left( \bar{\boldsymbol{\mu }} \right) Q_s\left( \bar{\boldsymbol{\mu }} \right),
    $ where
    \begin{equation}\label{eq10}
        f\left( \bar{\boldsymbol{\mu }} \right) =\mathbb{P} \left( H_s|\boldsymbol{\mu } _*=\bar{\boldsymbol{\mu }} \right) =\int_{\boldsymbol{\mu }}{\mathbb{P} \left( H_s|\boldsymbol{\mu } _s=\boldsymbol{\mu } \right) \mathbb{P} \left( \boldsymbol{\mu } _s=\boldsymbol{\mu } |\boldsymbol{\mu } _*=\bar{\boldsymbol{\mu }} \right) d\boldsymbol{\mu }}.
    \end{equation}

    We first calculate $f\left( \bar{\boldsymbol{\mu }} \right)$.
    By \eqref{eq10}, we have
    $$
        \begin{aligned}
            f\left( \bar{\boldsymbol{\mu}} \right)
             & =\int_{\boldsymbol{\mu }}{\prod_{t=1}^n{\varphi \left( r_{A_t}\left( t \right) ;\boldsymbol{b}_{A_t}\left( t \right) ^T\boldsymbol{\mu },v^2 \right)}\cdot \varPhi \left( \boldsymbol{\mu };\bar{\boldsymbol{\mu}},v^2\varSigma _* \right) d\boldsymbol{\mu }}
            \\
             & \propto \int_{\boldsymbol{\mu }}{\exp \left\{ -\frac{1}{2v^2}\left[ \sum_{t=1}^n{\left( r_{A_t}\left( t \right) -\boldsymbol{b}_{A_t}\left( t \right) ^T\boldsymbol{\mu } \right) ^2}+\left( \boldsymbol{\mu }-\bar{\boldsymbol{\mu}} \right) ^T\varSigma _{*}^{-1}\left( \boldsymbol{\mu }-\bar{\boldsymbol{\mu}} \right) \right] \right\} d\boldsymbol{\mu }}
            \\
             & \propto \exp \left\{ -\frac{1}{2v^2}\left[ \bar{\boldsymbol{\mu}}^T\varSigma _{*}^{-1}\bar{\boldsymbol{\mu}} \right] \right\}
            \\
             & \quad \cdot \int_{\boldsymbol{\mu }}{\exp \left\{ -\frac{1}{2v^2}\left[ -2\left( \sum_{t=1}^n{r_{A_t}\left( t \right) \boldsymbol{b}_{A_t}\left( t \right) ^T}+\bar{\boldsymbol{\mu}}^T\varSigma _{*}^{-1} \right) \boldsymbol{\mu }+\boldsymbol{\mu }^T\left( \sum_{t=1}^n{\boldsymbol{b}_{A_t}\left( t \right) \boldsymbol{b}_{A_t}\left( t \right) ^T}+\varSigma _{*}^{-1} \right) \boldsymbol{\mu } \right] \right\} d\boldsymbol{\mu }}.
        \end{aligned}
    $$
    To simplify the symbol, let
    $$
        G=\sum_{t=1}^n{\boldsymbol{b }_{A_t}\left( t \right) \boldsymbol{b }_{A_t}\left( t \right) ^T}+\varSigma _{*}^{-1},\  \boldsymbol{\xi} =G^{-1}\left( \sum_{t=1}^n{r_{A_t}\left( t \right) \boldsymbol{b }_{A_t}\left( t \right) ^T}+\bar{\boldsymbol{\mu }}^T\varSigma _{*}^{-1} \right) ^T.
    $$
    Then, we get
    $$
        \begin{aligned}
            f\left( \bar{\boldsymbol{\mu}} \right) & \propto \exp \left\{ -\frac{1}{2v^2}\left[ \bar{\boldsymbol{\mu}}^T\varSigma _{*}^{-1}\bar{\boldsymbol{\mu}} \right] \right\} \cdot \int_{\boldsymbol{\mu }}{\exp \left\{ -\frac{1}{2v^2}\left[ -2\boldsymbol{\xi }^TG\boldsymbol{\mu }+\boldsymbol{\mu }^TG\boldsymbol{\mu } \right] \right\} d\boldsymbol{\mu }}
            \\
                                                   & =\exp \left\{ -\frac{1}{2v^2}\left[ \bar{\boldsymbol{\mu}}^T\varSigma _{*}^{-1}\bar{\boldsymbol{\mu}}-\boldsymbol{\xi }^TG\boldsymbol{\xi } \right] \right\} \cdot \int_{\boldsymbol{\mu }}{\exp \left\{ -\frac{1}{2v^2}\left( \boldsymbol{\mu }-\boldsymbol{\xi } \right) ^TG\left( \boldsymbol{\mu }-\boldsymbol{\xi } \right) \right\} d\boldsymbol{\mu }}
            \\
                                                   & \propto \exp \left\{ -\frac{1}{2v^2}\left[ \bar{\boldsymbol{\mu}}^T\varSigma _{*}^{-1}\bar{\boldsymbol{\mu}}-\boldsymbol{\xi }^TG\boldsymbol{\xi } \right] \right\} .
        \end{aligned}
    $$

    Next, we derive $Q_{s+1}$ according to  $Q_{s+1}\left( \bar{\boldsymbol{\mu }} \right)  \propto f\left( \bar{\boldsymbol{\mu }} \right) Q_s\left( \bar{\boldsymbol{\mu }} \right)$, that is,
    $$
        \begin{aligned}
            Q_{s+1}\left( \bar{\boldsymbol{\mu }} \right) & \propto f\left( \bar{\boldsymbol{\mu }} \right) Q_s\left( \bar{\boldsymbol{\mu }} \right)
            \\
                                                          & \propto \exp \left\{ -\frac{1}{2v^2}\left[ \bar{\boldsymbol{\mu }}^T\varSigma _{*}^{-1}\bar{\boldsymbol{\mu }}-\boldsymbol{\xi} ^TG\boldsymbol{\xi} \right] \right\} \cdot \exp \left\{ -\frac{1}{2v^2}\left[ \left( \bar{\boldsymbol{\mu }}-\boldsymbol{\mu } _{Q,s} \right) ^T\varSigma _{Q,s}^{-1}\left( \bar{\boldsymbol{\mu }}-\boldsymbol{\mu } _{Q,s} \right) \right] \right\}
            \\
                                                          & \propto \exp \left\{ -\frac{1}{2v^2}\left[ \bar{\boldsymbol{\mu }}^T\varSigma _{*}^{-1}\bar{\boldsymbol{\mu }}-\boldsymbol{\xi} ^TG\boldsymbol{\xi} +\bar{\boldsymbol{\mu }}^T\varSigma _{Q,s}^{-1}\bar{\boldsymbol{\mu }}-2\bar{\boldsymbol{\mu }}^T\varSigma _{Q,s}^{-1}\boldsymbol{\mu } _{Q,s} \right] \right\}.
        \end{aligned}
    $$
    Denote $Y=\sum_{i=t}^n{r_{A_t}\left( t \right) \boldsymbol{b }_{A_t}\left( t \right)}$. Note that
    $$
        \begin{aligned}
            \boldsymbol{\xi} ^TG\boldsymbol{\xi} & =\left( Y^T+\bar{\boldsymbol{\mu }}^T\varSigma _{*}^{-1} \right) G^{-1}\left( Y^T+\bar{\boldsymbol{\mu }}^T\varSigma _{*}^{-1} \right) ^T
            \\
                                                 & =Y^TG^{-1}Y+2Y^TG^{-1}\varSigma _{*}^{-1}\bar{\boldsymbol{\mu }}+\bar{\boldsymbol{\mu }}^T\varSigma _{*}^{-1}G^{-1}\varSigma _{*}^{-1}\bar{\boldsymbol{\mu }} .
        \end{aligned}
    $$
    Then,
    $$
        \begin{aligned}
                    & Q_{s+1}\left( \bar{\boldsymbol{\mu }} \right)                                                                                                                                                                                                                                                                                       \\ \propto& \exp \left\{ -\frac{1}{2v^2}\left[ \bar{\boldsymbol{\mu }}^T\varSigma _{*}^{-1}\bar{\boldsymbol{\mu }}+\bar{\boldsymbol{\mu }}^T\varSigma _{Q,s}^{-1}\bar{\boldsymbol{\mu }}-2\bar{\boldsymbol{\mu }}^T\varSigma _{Q,s}^{-1}\boldsymbol{\mu } _{Q,s}-2Y^TG^{-1}\varSigma _{*}^{-1}\bar{\boldsymbol{\mu }}-\bar{\boldsymbol{\mu }}^T\varSigma _{*}^{-1}G^{-1}\varSigma _{*}^{-1}\bar{\boldsymbol{\mu }} \right] \right\}
            \\
            \propto & \exp \left\{ -\frac{1}{2v^2}\left[ \bar{\boldsymbol{\mu }}^T\left( \varSigma _{*}^{-1}+\varSigma _{Q,s}^{-1}-\varSigma _{*}^{-1}G^{-1}\varSigma _{*}^{-1} \right) \bar{\boldsymbol{\mu }}-2\left( {\boldsymbol{\mu } _{Q,s}}^T\varSigma _{Q,s}^{-1}+Y^TG^{-1}\varSigma _{*}^{-1} \right) \bar{\boldsymbol{\mu }} \right] \right\} .
        \end{aligned}
    $$
    Let
    $$
        W=\varSigma _{*}^{-1}+\varSigma _{Q,s}^{-1}-\varSigma _{*}^{-1}G^{-1}\varSigma _{*}^{-1},\  \boldsymbol{\eta } =W^{-1}\left( {\boldsymbol{\mu } _{Q,s}}^T\varSigma _{Q,s}^{-1}+Y^TG^{-1}\varSigma _{*}^{-1} \right) ^T.
    $$
    Since
    $$
        W=\varSigma _{Q,s}^{-1}+\varSigma _{*}^{-1}\left( \varSigma _*-G^{-1} \right) \varSigma _{*}^{-1},
    $$
    and $\varSigma _*-G^{-1}$ is positive    semi-definite matrix,  we know that $W$ is positive    definite matrix.

    Furthermore,
    $$
        \begin{aligned}
            Q_{s+1}\left( \bar{\boldsymbol{\mu }} \right) & \propto \exp \left\{ -\frac{1}{2v^2}\left[ \bar{\boldsymbol{\mu }}^TW\bar{\boldsymbol{\mu }}-2\boldsymbol{\eta } ^TW\bar{\boldsymbol{\mu }} \right] \right\}
            \\
                                                          & \propto \exp \left\{ -\frac{1}{2v^2}\left[ \bar{\boldsymbol{\mu }}^TW\bar{\boldsymbol{\mu }}-2\boldsymbol{\eta } ^TW\bar{\boldsymbol{\mu }}+\boldsymbol{\eta } ^TW\boldsymbol{\eta } \right] \right\}
            \\
                                                          & \propto \varPhi \left( \bar{\boldsymbol{\mu}};\boldsymbol{\eta },v^2W^{-1} \right).
        \end{aligned}
    $$
    Note that
    $$
        G^{-1}\varSigma _{*}^{-1}=\left( \varSigma _*\sum_{t=1}^n{\boldsymbol{b}_{A_t}\left( t \right) \boldsymbol{b}_{A_t}\left( t \right) ^T}+I_d \right) ^{-1},\quad
        \varSigma _{*}^{-1}G^{-1}\varSigma _{*}^{-1}=\left[ \varSigma _*\left( \sum_{t=1}^n{\boldsymbol{b}_{A_t}\left( t \right) \boldsymbol{b}_{A_t}\left( t \right) ^T} \right) \varSigma _*+\varSigma _* \right] ^{-1},
    $$
    and we obtain
    $$
        \varSigma _{Q,s+1}=W^{-1}=\left[ \varSigma _{Q,s}^{-1}+\varSigma _{*}^{-1}-\left( \varSigma _*\left( \sum_{t=1}^n{\boldsymbol{b }_{A_t}\left( t \right) \boldsymbol{b }_{A_t}\left( t \right) ^T} \right) \varSigma _*+\varSigma _* \right) ^{-1} \right] ^{-1},
    $$
    $$
        \boldsymbol{\mu } _{Q,s+1}=\boldsymbol{\eta } =\varSigma _{Q,s+1}\left[ {\boldsymbol{\mu } _{Q,s}}^T\varSigma _{Q,s}^{-1}+\left( \sum_{t=1}^n{r_{A_t}\left( t \right) \boldsymbol{b }_{A_t}\left( t \right)} \right) ^T\left( \varSigma _*\sum_{t=1}^n{\boldsymbol{b }_{A_t}\left( t \right) \boldsymbol{b }_{A_t}\left( t \right) ^T}+I_d \right) ^{-1} \right] ^T.
    $$

\end{proof}

\subsection{Proof of Lemma \ref{theo1}}
\begin{proof}
       Let $\hat{\boldsymbol{\mu }}\left( t \right) \in \mathbb{R}^d$ be the  Maximum A Posteriori (MAP) estimate of $\boldsymbol{\mu }$ in round $t$, $\tilde{\boldsymbol{\mu }}(t) \in \mathbb{R}^d$ be the posterior sample in round $t$, and $H_t$ denote  the union of  history  and the contexts in round $t$. Note that in posterior sampling, $\mathbb{P} \left( \tilde{\boldsymbol{\mu}}\left( t \right) =\bar{\boldsymbol{\mu}}\mid H_t \right) =\mathbb{P} \left( \boldsymbol{\mu }=\bar{\boldsymbol{\mu}}\mid H_t \right) $
    for all $\bar{\boldsymbol{\mu}}$, and $
        \tilde{\boldsymbol{\mu }}(t)\sim \mathcal{N} \left( \hat{\boldsymbol{\mu }}\left( t \right) ,v^2B\left( t \right) ^{-1} \right)$, where $B(1)=\varSigma _*^{-1}$.

    Denote $\theta _i\left( t \right) =\boldsymbol{b }_i\left( t \right) ^T\boldsymbol{\mu }$, $ \hat{\theta}_i\left( t \right) =\boldsymbol{b }_i\left( t \right) ^T\hat{\boldsymbol{\mu }}\left( t \right) $ and $\tilde{\theta}_i\left( t \right) =\boldsymbol{b }_i\left( t \right) ^T\tilde{\boldsymbol{\mu }}\left( t \right) .$ Then we obtain that
    $$\tilde{\theta}_i\left( t \right) \sim \mathcal{N} \left( \hat{\theta}_i\left( t \right) ,v^2s_{i}\left( t \right)^{2} \right) , \  s_{i}\left( t \right)^{2} =\boldsymbol{b }_i\left( t \right) ^TB\left( t \right) ^{-1}\boldsymbol{b }_i\left( t \right),$$
    and
    $\mathbb{P} \left( \tilde{\theta}_i\left( t \right) =\bar{\theta}_i\mid H_t \right) =\mathbb{P} \left( \theta _i\left( t \right) =\bar{\theta}_i\mid H_t \right) $ for all $\bar{\theta}_i$ and $i\in [k]$.

    Accordingly, a high-probability confidence interval of $\theta _i\left( t \right)$  is $C_i\left( t \right) =v s_i\left( t \right) \sqrt{2\log \left( \frac{1}{\delta} \right)}$, where $\delta>0$ is the confidence level. Let
    $$
        E_t=\left\{ \forall i\in [k]:\left| \theta _i(t)-\hat{\theta}_i(t) \right|\leqslant C_i\left( t \right) \right\}
    $$
    be the event that all confidence intervals in round $t$ hold.

  Fix round $t$, and the regret can be decomposed as
    $$
        \begin{aligned}
            \mathbb{E} \left[ \theta _{A_{t}^*  }\left( t \right) -\theta _{A_t}\left( t \right) \right] & =\mathbb{E} \left[ \mathbb{E} \left[ \theta _{A_{t}^*  }\left( t \right) -\theta _{A_t}\left( t \right) \mid H_t \right] \right]                                                                                                                                                                                                 \\
                                                                                                       & =\mathbb{E} \left[ \mathbb{E} \left[ \theta _{A_{t}^*  }\left( t \right) -\hat{\theta}_{A_{t}^*  }\left( t \right) -C_{A_{t}^*  }\left( t \right) \mid H_t \right] \right] +\mathbb{E} \left[ \mathbb{E} \left[ \hat{\theta}_{A_t}\left( t \right) +C_{A_t}\left( t \right) -\theta _{A_t}\left( t \right) \mid H_t \right] \right].
        \end{aligned}
    $$
    The first equality is an application of the tower rule. The second equality holds because $A_{t} \mid H_t$ and $A_{t}^*   \mid H_t$ have the same distributions, and $\hat{\theta}_i(t)$ and $C_i(t)$, $i \in [k]$, are deterministic given history $H_t$.

    We start with the first term in the decomposition. Fix history $H_t$, then we introduce event $E_t$ and get
    $$
        \begin{aligned}
            \mathbb{E} \left[ \theta _{A_{t}^*  }\left( t \right) -\hat{\theta}_{A_{t}^*  }\left( t \right) -C_{A_{t}^*  }\left( t \right) \mid H_t \right] & =\mathbb{E} \left[ \theta _{A_{t}^*  }\left( t \right) -\hat{\theta}_{A_{t}^*  }\left( t \right) \mid H_t \right] -\mathbb{E} \left[ C_{A_{t}^*  }\left( t \right) \mid H_t \right] \\
                                                                                                                                                      & \leqslant \mathbb{E} \left[ \left( \theta _{A_{t}^*  }\left( t \right) -\hat{\theta}_{A_{t}^*  }\left( t \right) \right) \mathbf{1}\left\{ \bar{E}_t \right\} \mid H_t \right] ,
        \end{aligned}
    $$
    where the inequality follows from the observation that
    $$
        \mathbb{E} \left[ \left( \theta _{A_{t}^*  }\left( t \right) -\hat{\theta}_{A_{t}^*  }\left( t \right) \right) \mathbf{1}\left\{ E_t \right\} \mid H_t \right] \leqslant \mathbb{E} \left[ C_{A_{t}^*  }\left( t \right) \mid H_t \right] .
    $$
    Since $
        \theta _i\left( t \right) -\hat{\theta}_i\left( t \right) \mid H_t\sim \mathcal{N} \left( 0,v^2s_i\left( t \right) ^2 \right) , i\in \left[ k \right]
    $, we further have
    \begin{equation}\label{eq3}
        \begin{aligned}
            \mathbb{E} \left[ \left( \theta _{A_{t}^*  }\left( t \right) -\hat{\theta}_{A_{t}^*  }\left( t \right) \right) \mathbf{1}\left\{ \bar{E}_t \right\} \mid H_t \right] & \leqslant \sum_{i=1}^k{\frac{1}{\sqrt{2\pi v^2s_i\left( t \right) ^2}}}\int_{x=C_i\left( t \right)}^{\infty}{x}\exp \left[ -\frac{x^2}{2v^2s_i\left( t \right) ^2} \right] \mathrm{d}x                                \\
                                                                                                                                                                             & =\sum_{i=1}^k{-}\sqrt{\frac{v^2s_i\left( t \right) ^2}{2\pi}}\int_{x=C_i\left( t \right)}^{\infty}{\frac{\partial}{\partial x}}\left( \exp \left[ -\frac{x^2}{2v^2s_i\left( t \right) ^2} \right] \right) \mathrm{d}x \\
                                                                                                                                                                             & =\sum_{i=1}^k{\sqrt{\frac{v^2s_i\left( t \right) ^2}{2\pi}}}\delta \leqslant k\delta v\sqrt{\frac{1}{2\pi \lambda _{\min}}},                                                                                          \\
        \end{aligned}
    \end{equation}
    in which the last inequality is obtained by  \eqref{eq17}.

    For the second the term in the regret decomposition, we have
    $$
        \mathbb{E} \left[ \hat{\theta}_{A_t}\left( t \right) +C_{A_t}\left( t \right) -\theta _{A_t}\left( t \right) \mid H_t \right] \leqslant 2\mathbb{E} \left[ C_{A_t}\left( t \right) \mid H_t \right] +\mathbb{E} \left[ \left( \hat{\theta}_{A_t}\left( t \right) -\theta _{A_t}\left( t \right) \right) \mathbf{1}\left\{ \bar{E}_t \right\} \mid H_t \right] ,
    $$
    where the inequality follows from the observation that
    $$
        \mathbb{E} \left[ \left( \hat{\theta}_{A_t}\left( t \right) -\theta _{A_t}\left( t \right) \right) \mathbf{1}\left\{ E_t \right\} \mid H_t \right] \leqslant \mathbb{E} \left[ C_{A_t}\left( t \right) \mid H_t \right] .
    $$

    The other term is bounded as in \eqref{eq3}. Now we chain all inequalities for the regret in round $t$ and get
    $$
        \mathbb{E} \left[ \theta _{A_{t}^*  }\left( t \right) -\theta _{A_t}\left( t \right) \right] \leqslant 2\mathbb{E} \left[ C_{A_t}\left( t \right) \right] +k\delta v\sqrt{\frac{2}{\pi \lambda _{\min}}}.
    $$
    Therefore, the $n$-round Bayes regret is bounded as
    $$
        \mathbb{E} \left[ \sum_{t=1}^n{\theta _{A_{t}^*  }\left( t \right) -\theta _{A_t}\left( t \right)} \right] \leqslant 2\mathbb{E} \left[ \sum_{t=1}^n{C_{A_t}\left( t \right)} \right] +nk\delta v\sqrt{\frac{2}{\pi \lambda _{\min}}}.
    $$

    The last part is to bound $\mathbb{E} \left[ \sum_{t=1}^n{C_{A_t}\left( t \right)} \right]$ from above. 
    By \eqref{eq16}, we have
$$
\begin{aligned}
	\mathbb{E} \left[ \sum_{t=1}^n{C_{A_t}\left( t \right)} \right] &=\mathbb{E} \left[ \sum_{t=1}^n{vs_{A_t}\left( t \right) \sqrt{2\log \left( \frac{1}{\delta} \right)}} \right] =v\sqrt{2\log \left( \frac{1}{\delta} \right)}\mathbb{E} \left[ \sum_{t=1}^n{s_{A_t}\left( t \right)} \right]\\
	&\leqslant v\left( \sqrt{\frac{1}{\lambda _{\min}}}+\sqrt{\frac{n-1}{\vartheta}} \right) \sqrt{2\log \left( \frac{1}{\delta} \right)}.
\end{aligned}
$$
    Now we chain all inequalities and this completes the proof.

\end{proof}

\subsection{Proof of Lemma \ref{theo2}}
\begin{proof}
    Denote $\theta _i\left( t \right) =\boldsymbol{b }_i\left( t \right) ^T\boldsymbol{\mu }$, $\theta _{i,*}\left( t \right) =\boldsymbol{b}_i\left( t \right) ^T\boldsymbol{\mu }_*$ and   we obtain that
    $$
        \theta _i\left( t \right) \sim \mathcal{N} \left( \theta _{i,*}\left( t \right) ,v^2\sigma _i\left( t \right) ^2 \right) ,\quad \sigma _i\left( t \right) ^2=\boldsymbol{b}_i\left( t \right) ^T\varSigma _*\boldsymbol{b}_i\left( t \right) .
    $$
    First, we bound the regret when $\boldsymbol{\mu }$ is not close to $\boldsymbol{\mu }_*$. Let
    $$
        E_t=\left\{ \forall i\in [k]:\left| \theta _i\left( t \right) -\theta _{i,*}\left( t \right) \right|\leqslant c_i(t) \right\}
    $$
    be the event that $\boldsymbol{\mu }$ is close to $\boldsymbol{\mu }_*$, where $
        c_i\left( t \right) =v\sigma _i\left( t \right) \sqrt{2 \log \left( \frac{1}{\delta} \right) }
    $ is the corresponding confidence interval. Then
    $$
        \begin{aligned}
             & \quad \left| \mathbb{E} \left[ \theta _a\left( t \right) \right] -\mathbb{E} \left[ \theta _a\left( t \right) \mathbf{1}\left\{ E_t \right\} \right] \right|                                                                                                                                                                                          \\
             & =\left| \mathbb{E} \left[ (\theta _a\left( t \right) -\theta _{a,*}\left( t \right) )\mathbf{1}\left\{ \bar{E}_t \right\} \right] +\mathbb{E} \left[ \theta _{a,*}\left( t \right) \mathbf{1}\left\{ \bar{E}_t \right\} \right] \right|                                                                                                               \\
             & \leqslant \mathbb{E} \left[ \left| \theta _a\left( t \right) -\theta _{a,*}\left( t \right) \right|\mathbf{1}\left\{ \bar{E}_t \right\} \right] +\mathbb{E} \left[ \left| \theta _{a,*}\left( t \right) \right|\mathbf{1}\left\{ \bar{E}_t \right\} \right]                                                                                           \\
             & \leqslant 2\sum_{i=1}^k{\left\{ \frac{1}{\sqrt{2\pi v^2\sigma _i\left( t \right) ^2}}\int_{x=c_i\left( t \right)}^{\infty}{x}\exp \left[ -\frac{x^2}{2v^2\sigma _i\left( t \right) ^2} \right] \mathrm{d}x+\left| \theta _{i,*}\left( t \right) \right|\exp \left[ -\frac{c_i\left( t \right) ^2}{2v^2\sigma _i\left( t \right) ^2} \right] \right\}} \\
             & =2\sum_{i=1}^k{\left\{ -\sqrt{\frac{v^2\sigma _i\left( t \right) ^2}{2\pi}}\int_{x=c_i\left( t \right)}^{\infty}{\left( \exp \left[ -\frac{x^2}{2v^2\sigma _i\left( t \right) ^2} \right] \right) ^{\prime}}\mathrm{d}x+\left| \theta _{i,*}\left( t \right) \right|\delta \right\}}                                                                  \\
             & =2\delta \sum_{i=1}^k{\left\{ \sqrt{\frac{v^2\sigma _i\left( t \right) ^2}{2\pi}}+\left| \theta _{i,*}\left( t \right) \right| \right\}}                                                                                                                                                                                                              \\
             & \leqslant 2k\delta \left( \sqrt{\frac{v^2}{2\pi \lambda _{\min}}}+\left\| \boldsymbol{\mu }_* \right\|   \right) ,
        \end{aligned}
    $$
    where the last inequality is obtained  by   $
        \left| \theta _{i,*}\left( t \right) \right|\leqslant \left\| \boldsymbol{b}_i\left( t \right) \right\| \left\| \boldsymbol{\mu }_* \right\| \leqslant \left\| \boldsymbol{\mu }_* \right\|
    $.

    Now we apply this decomposition to both $\theta _{\hat{A}_t}\left( t \right)$ and $\theta _{\tilde{A}_t}\left( t \right)$ below, and get
    $$
        \mathbb{E} \left[ \sum_{t=1}^n{\theta _{\hat{A}_t}\left( t \right) -\theta _{\tilde{A}_t}\left( t \right)} \right] \leqslant \mathbb{E} \left[ \mathbf{1}\left\{ E_t \right\} \sum_{t=1}^n{\theta _{\hat{A}_t}\left( t \right) -\theta _{\tilde{A}_t}\left( t \right)} \right] +4nk\delta \left( \sqrt{\frac{v^2}{2\pi \lambda _{\min}}}+\left\| \boldsymbol{\mu }_* \right\|   \right) .
    $$


    The primary obstacle in bounding the aforementioned first term arises from the potential significant deviation in the posterior distributions of  $\hat{A}_t$ and $\tilde{A}_t$, contingent upon the divergence in their respective histories. Consequently, leveraging solely the difference in their prior means, denoted as $\varepsilon$, to constrain their divergence poses a formidable challenge.

    Similar to the analysis in the proof of Lemma 5 in \cite{metaTS}, we have that in round 1, the two TS algorithms behave differently, on average over the posterior samples, in $\sum_{i=1}^k\left|\mathbb{P}\left(\hat{A}_1=i\right)-\mathbb{P}\left(\tilde{A}_1=i\right)\right|$  fraction of runs.
    We bound the difference of their future rewards trivially by $$
\theta _{\hat{A}_t}\left( t \right) -\theta _{\tilde{A}_t}\left( t \right) \leqslant \left| \theta _{\tilde{A}_t,*}\left( t \right) \right|+\left| \theta _{\hat{A}_t,*}\left( t \right) \right|+c_{\hat{A}_t}(t)+c_{\tilde{A}_t}(t)\leqslant 2\left( \left\| \boldsymbol{\mu }_* \right\| +c_0 \right) ,
$$
    where $
        c_0=v\sqrt{\frac{2 \log \left( \frac{1}{\delta} \right) }{\lambda _{\min}}}
    $.

    Now we apply this bound from round 2 to $n$, conditioned on both algorithms having the same history distributions, and get
$$
\mathbb{E} \left[ 1\left\{ E_t \right\} \sum_{t=1}^n{\theta _{\hat{A}_t}\left( t \right) -\theta _{\tilde{A}_t}\left( t \right)} \right] \leqslant 2\left( \left\| \boldsymbol{\mu }_* \right\| +c_0 \right) \sum_{t=1}^n{\left\{ \max_{h\in \mathcal{H} _t} \sum_{i=1}^k{\left| \mathbb{P} \left( \hat{A}_t=i\mid \hat{H}_t=h \right) -\mathbb{P} \left( \tilde{A}_t=i\mid \tilde{H}_t=h \right) \right|} \right\}},
$$
    where $\hat{H}_t$ is the history for $\hat{A}_t, \tilde{H}_t$ is the history for $\tilde{A}_t$, and $\mathcal{H}_t$ is the set of all possible histories in round $t$. Finally, we bound the last term above using $\varepsilon$.

    Fix round $t$ and history $h \in \mathcal{H}_t$.    Let  $
        \dot{\boldsymbol{\mu}}\left( t \right)
    $ and $
        \ddot{\boldsymbol{\mu}}\left( t \right)
    $ be the posterior samples of two TS algorithm in round $t$. Let $p(\theta )=\mathbb{P} \left( \hat{\theta}\left( t \right) =\theta \mid \hat{H}_t=h \right) $ and $
        q(\theta )=\mathbb{P} \left( \tilde{\theta}\left( t \right) =\theta \mid \tilde{H}_t=h \right)
    $, where  the $i$-elements of $\hat{\theta}\left( t \right)$ and $\tilde{\theta}\left( t \right)$ are $
        \boldsymbol{b}_i\left( t \right) ^T\dot{\boldsymbol{\mu}}\left( t \right) $ and $ \boldsymbol{b}_i\left( t \right) ^T\ddot{\boldsymbol{\mu}}\left( t \right)
    $.
    Then, since the pulled arms are deterministic functions of their posterior samples, we have
    $$
        \sum_{i=1}^k\left|\mathbb{P}\left(\hat{A}_t=i \mid \hat{H}_t=h\right)-\mathbb{P}\left(\tilde{A}_t=i \mid \tilde{H}_t=h\right)\right| \leqslant \int_\theta|p(\theta)-q(\theta)| \mathrm{d} \theta .
    $$
    Moreover, $p(\theta)=\prod_{i=1}^k p\left(\theta_i\right)$ and $q(\theta)=\prod_{i=1}^k q\left(\theta_i\right)$ when the reward noise and prior distributions factor over individual arms, and according to the proof of Lemma 5 in \cite{metaTS}, we have
    $$
        \int_\theta|p(\theta)-q(\theta)| \mathrm{d} \theta \leqslant \sum_{i=1}^k \int_{\theta_i}\left|p\left(\theta_i\right)-q\left(\theta_i\right)\right| \mathrm{d} \theta_i .
    $$
    Note that  $
        p\left( \theta _i \right) =\varphi \left( \theta _i;\boldsymbol{b}_i\left( t \right) ^T\hat{\boldsymbol{\mu}}\left( t \right) ,v^2s_i\left( t \right) ^2 \right)
    $ and $
        q\left( \theta _i \right) =\varphi \left( \theta _i;\boldsymbol{b}_i\left( t \right) ^T\tilde{\boldsymbol{\mu}}\left( t \right) ,v^2s_i\left( t \right) ^2 \right)
    $, where
    $$
        \tilde{\boldsymbol{\mu}}\left( t \right) =B\left( t \right) ^{-1}\left[ \varSigma _{*}^{-1}\tilde{\boldsymbol{\mu}}+\sum_{\tau =1}^{t-1}{\boldsymbol{b}_{A_{\tau}}\left( \tau \right) r_{A_{\tau}}\left( \tau \right)} \right] ,\quad
        \hat{\boldsymbol{\mu}}\left( t \right) =B\left( t \right) ^{-1}\left[ \varSigma _{*}^{-1}\hat{\boldsymbol{\mu}}+\sum_{\tau =1}^{t-1}{\boldsymbol{b}_{A_{\tau}}\left( \tau \right) r_{A_{\tau}}\left( \tau \right)} \right] .
    $$
    Then, under the assumption that $
        \left\| \hat{\boldsymbol{\mu}}-\tilde{\boldsymbol{\mu}} \right\| \leqslant \varepsilon $
    , each above integral is bounded as
    $$
        \begin{aligned}
            \int_{\theta _i}{\left| p\left( \theta _i \right) -q\left( \theta _i \right) \right|}\mathrm{d}\theta _i & \leqslant \frac{2}{\sqrt{2\pi v^2s_i\left( t \right) ^2}}\left| \boldsymbol{b}_i\left( t \right) ^T\hat{\boldsymbol{\mu}}\left( t \right) -\boldsymbol{b}_i\left( t \right) ^T\tilde{\boldsymbol{\mu}}\left( t \right) \right|
            \\
                                                                                                                     & =\frac{2}{\sqrt{2\pi v^2s_i\left( t \right) ^2}}\left| \boldsymbol{b}_i\left( t \right) ^TB\left( t \right) ^{-1}\varSigma _{*}^{-1}\left( \hat{\boldsymbol{\mu}}-\tilde{\boldsymbol{\mu}} \right) \right|
            \\
                                                                                                                     & \leqslant \sqrt{\frac{2}{\pi v^2}}\frac{\left\| \boldsymbol{b}_i\left( t \right) ^TB\left( t \right) ^{-1} \right\|}{\sqrt{s_i\left( t \right) ^2}}\left\| \varSigma _{*}^{-1} \right\| \left\| \hat{\boldsymbol{\mu}}-\tilde{\boldsymbol{\mu}} \right\|
            \\
                                                                                                                     & \leqslant \sqrt{\frac{2}{\pi v^2}}\frac{\left\| \boldsymbol{b}_i\left( t \right) ^TB\left( t \right) ^{-1} \right\|}{\sqrt{s_i\left( t \right) ^2}}\sqrt{\lambda _{\max}}\varepsilon.
        \end{aligned}
    $$
    The first inequality holds for any two shifted non-negative unimodal functions, with maximum $
        \frac{1}{\sqrt{2\pi v^2s_i\left( t \right) ^2}}
    $.
    Finally, we need to bound $\frac{\left\| \boldsymbol{b}_i\left( t \right) ^TB\left( t \right) ^{-1} \right\|}{\sqrt{s_i\left( t \right) ^2}}$. Denote by $
        \boldsymbol{y}=B\left( t \right) ^{-1}\boldsymbol{b}_i\left( t \right)
    $.  According to Rayleigh theorem, 
    $$
        \frac{\left\| \boldsymbol{b}_i\left( t \right) ^TB\left( t \right) ^{-1} \right\|}{\sqrt{s_i\left( t \right) ^2}}=\sqrt{\frac{\boldsymbol{y}^T\boldsymbol{y}}{\boldsymbol{y}^TB\left( t \right) \boldsymbol{y}}}\leqslant \sqrt{\frac{1}{\lambda _{\min}\left( B\left( t \right) \right)}}.
    $$  Thus, $
\int_{\theta}{|}p(\theta )-q(\theta )|\mathrm{d}\theta \leqslant k\sqrt{\frac{2\lambda _{\max}}{\pi v^2\lambda _{\min}\left( B\left( t \right) \right)}}\varepsilon .
$ By Lemma \ref{lem10}, we get 
$$
\begin{aligned}
	\sum_{t=1}^n{\left\{ \sum_{i=1}^k{\left| \mathbb{P} \left( \hat{A}_t=i\mid \hat{H}_t=h \right) -\mathbb{P} \left( \tilde{A}_t=i\mid \tilde{H}_t=h \right) \right|} \right\}}&\leqslant k\varepsilon \sqrt{\frac{2\lambda _{\max}}{\pi v^2}}\sum_{t=1}^n{\sqrt{\frac{1}{\lambda _{\min}\left( B\left( t \right) \right)}}}\\
	&\leqslant k\varepsilon \left( \sqrt{\frac{1}{\lambda _{\min}}}+\sqrt{\frac{n-1}{\vartheta}} \right) \sqrt{\frac{2\lambda _{\max}}{\pi v^2}}.\\
\end{aligned}
$$

    This completes the proof.
\end{proof}

\subsection{Proof of Lemma \ref{theo3G}} \label{prooftheo3G}
\begin{proof}
    The key idea in this proof is that $
        \boldsymbol{\mu }_*|H_{1:s}\sim \mathcal{N} \left( \boldsymbol{\mu }_{Q,s},v^2\varSigma _{Q,s} \right)
    $.
    By SVD decomposition of $\varSigma _{Q,s}$, we can get $
        \varSigma _{Q,s}=UDU^{T}
    $, where $U$ is an orthogonal matrix, $
        D=\mathrm{diag}\left( \left( \lambda _{s,i} \right) _{i=1}^{d} \right)
    $ and $\lambda _{s,i}$ is the $i$-th largest eigenvalue of $\varSigma _{Q,s}$. Let $
        \boldsymbol{y}=U^T\left( \boldsymbol{\mu }_*-\boldsymbol{\mu }_{Q,s} \right) $ and it follows that $
        \boldsymbol{y}|H_{1:s}\sim \mathcal{N} \left( \boldsymbol{0},v^2D \right)
    $. To simplify notation, let $
        \boldsymbol{y}=\left( y_i \right) _{i=1}^{d}
    $. Note that $
        \left\| \boldsymbol{y} \right\|=\left\| \boldsymbol{\mu }_*-\boldsymbol{\mu }_{Q,s} \right\|
    $, then   we have that for any $\varepsilon>0$,
    \begin{equation}\label{eq4}
        \begin{aligned}
            \mathbb{P} \left( \left\| \boldsymbol{\mu }_*-\boldsymbol{\mu }_{Q,s} \right\| >\varepsilon \mid H_{1:s-1} \right) & = \mathbb{P} \left( \left\| \boldsymbol{y} \right\| >\varepsilon \mid H_{1:s-1} \right) \leqslant \mathbb{P} \left( \left\| \boldsymbol{y} \right\| _{\infty}>\frac{\varepsilon}{\sqrt{d}}\mid H_{1:s-1} \right)
            \\
                                                                                                                               &
            \leqslant \sum_{i=1}^d{\mathbb{P} \left( \left| y_i \right|>\frac{\varepsilon}{\sqrt{d}}\mid H_{1:s-1} \right)}.
        \end{aligned}
    \end{equation}

    Since $y_i|H_{1:s}\sim \mathcal{N} \left( 0,v^2\lambda _{s,i} \right) $,
    $$
        \mathbb{P} \left( \left| y_i \right|>\frac{\varepsilon}{\sqrt{d}}\mid H_{1:s-1} \right) \leqslant 2\exp \left[ -\frac{\left( \frac{\varepsilon}{\sqrt{d}} \right) ^2}{v^2\lambda _{s,i}} \right] \leqslant 2\exp \left[ -\frac{\varepsilon ^2}{dv^2\lambda _{s,1}} \right] .
    $$


    Note that
    $$
        \begin{aligned}
             \lambda _{\max}\left( \varSigma _{Q,s+1} \right) & =\frac{1}{ \lambda_{\min}\left( \varSigma _{Q,s}^{-1}+\varSigma _{*}^{-1}-\left( \varSigma _*\left( \sum_{t=1}^n{\boldsymbol{b}_{A_t}\left( t \right) \boldsymbol{b}_{A_t}\left( t \right) ^T} \right) \varSigma _*+\varSigma _* \right) ^{-1} \right)}                                      \\
                                                          & \leqslant \frac{1}{ \lambda _{\min}\left( \varSigma _{Q,s}^{-1} \right) + \lambda _{\min}\left( \varSigma _{*}^{-1}-\left( \varSigma _*\left( \sum_{t=1}^n{\boldsymbol{b}_{A_t}\left( t \right) \boldsymbol{b}_{A_t}\left( t \right) ^T} \right) \varSigma _*+\varSigma _* \right) ^{-1} \right)} \\
                                                          & <\frac{1}{ \lambda _{\min}\left( \varSigma _{Q,s}^{-1} \right)}= \lambda _{\max}\left( \varSigma _{Q,s} \right).
        \end{aligned}
    $$
    The second inequality follows from the fact  that $\varSigma _{*}^{-1}-\left( \varSigma _*\left( \sum_{t=1}^n{\boldsymbol{b}_{A_t}\left( t \right) \boldsymbol{b}_{A_t}\left( t \right) ^T} \right) \varSigma _*+\varSigma _* \right) ^{-1}$ is positive definite.

    Moreover, we obtain
    $$
        \begin{aligned}
             \lambda _{\max}\left( \varSigma _{Q,s+1} \right) & \leqslant \frac{1}{ \lambda _{\min}\left( \varSigma _{Q,s}^{-1} \right) + \lambda _{\min}\left( \varSigma _{*}^{-1}-\left( \varSigma _*\left( \sum_{t=1}^n{\boldsymbol{b}_{A_t}\left( t \right) \boldsymbol{b}_{A_t}\left( t \right) ^T} \right) \varSigma _*+\varSigma _* \right) ^{-1} \right)} \\
                                                          & \leqslant \frac{1}{ \lambda _{\min}\left( \varSigma _{Q,s}^{-1} \right) +\frac{1}{2} \lambda _{\min}\left( \varSigma _{*}^{-1} \right)}=\frac{ \lambda _{\max}\left( \varSigma _{Q,s} \right)}{1+\frac{\lambda _{\min}}{2}\cdot \lambda _{\max}\left( \varSigma _{Q,s} \right)}                                \\
                                                          & \leqslant \frac{ \lambda _{\max}\left( \varSigma _{Q,s} \right)}{1+\frac{\lambda _{\min}}{2}\cdot \lambda _{\max}\left( \varSigma _{Q,s+1} \right)}.
        \end{aligned}
    $$

    Based on the above equation, we get
    $$
        \begin{aligned}
                        & \lambda _{\min} \cdot \lambda_{\max}\left( \varSigma _{Q,s+1} \right) +1\leqslant \sqrt{1+2\lambda _{\min}\cdot \lambda _{\max}\left( \varSigma _{Q,s} \right)}\leqslant \frac{7}{8}\lambda _{\min} \cdot\lambda _{\max}\left( \varSigma _{Q,s} \right) +\frac{101}{100}
            \\
            \Rightarrow &  \lambda _{\max}\left( \varSigma _{Q,s+1} \right) \leqslant \frac{7}{8} \lambda _{\max}\left( \varSigma _{Q,s} \right) +\frac{1}{100\lambda _{\min}},
        \end{aligned}
    $$
    and
    $$
        \lambda_{s,1}=\lambda _{\max}\left( \varSigma _{Q,s} \right) \leqslant \left( \lambda _{\max}\left( \varSigma _{Q} \right)-\frac{2}{175\lambda _{\min}} \right) \left( \frac{7}{8} \right) ^{s-1}+\frac{2}{175\lambda _{\min}}.
    $$

    Thus,
    \begin{equation}\label{eq5}
        \mathbb{P} \left( \left| y_i \right|>\frac{\varepsilon}{\sqrt{d}}\mid H_{1:s-1} \right) \leqslant 2\exp \left[ -\frac{\varepsilon ^2}{dv^2\lambda _{s,1}} \right] \leqslant 2\exp \left[ -\frac{\varepsilon ^2}{dv^2\left( \left(  \lambda _{\max}\left( \varSigma _{Q} \right)-\frac{2}{175\lambda _{\min}} \right) \left( \frac{7}{8} \right) ^{s-1}+\frac{2}{175\lambda _{\min}} \right)} \right] .
    \end{equation}

    Combined with \eqref{eq4} and \eqref{eq5}, we can get
    $$
        \mathbb{P} \left( \left\| \boldsymbol{\mu }_*-\boldsymbol{\mu }_{Q,s} \right\| >\varepsilon \mid H_{1:s-1} \right) \leqslant2d\exp \left[ -\frac{\varepsilon ^2}{dv^2\left( \left(  \lambda _{\max}\left( \varSigma _{Q} \right)-\frac{2}{175\lambda _{\min}} \right) \left( \frac{7}{8} \right) ^{s-1}+\frac{2}{175\lambda _{\min}} \right)} \right].
    $$

    Now we choose $
        \varepsilon_s =v\sqrt{d\left( \left(  \lambda _{\max}\left( \varSigma _{Q} \right)-\frac{2}{175\lambda _{\min}} \right) \left( \frac{7}{8} \right) ^{s-1}+\frac{2}{175\lambda _{\min}} \right) \log \left( \frac{2d}{\delta} \right)}
    $
    and get that
    \begin{equation}
        \mathbb{P} \left( \left\| \boldsymbol{\mu }_*-\boldsymbol{\mu }_{Q,s} \right\| >\varepsilon_s \mid H_{1:s-1} \right) \leqslant {\delta}
    \end{equation}
    for any task $s$ and history $H_{1: s-1}$. It follows that
    $$
        \mathbb{P} \left( \bigcup_{s=1}^m{\left\{ \left\| \boldsymbol{\mu }_*-\boldsymbol{\mu }_{Q,s} \right\| >\varepsilon_s \right\}} \right) \leqslant \sum_{s=1}^m{\mathbb{P}}\left( \left\| \boldsymbol{\mu }_*-\boldsymbol{\mu }_{Q,s} \right\| >\varepsilon_s \right) =\sum_{s=1}^m{\mathbb{E}}\left[ \mathbb{P} \left( \left\| \boldsymbol{\mu }_*-\boldsymbol{\mu }_{Q,s} \right\| >\varepsilon_s \mid H_{1:s-1} \right) \right] \leqslant m\delta.
    $$
\end{proof}

\subsection{Proof of Theorem \ref{theo_mainG}}
\begin{proof}
    First, we bound the magnitude of $\boldsymbol{\mu }_*.$ Specifically
    since $\boldsymbol{\mu }_*\sim \mathcal{N}(\boldsymbol{\mu }_Q,v^2 \varSigma _Q )$, according to the proof of Lemma \ref{theo3G}, we have that
    \begin{equation}\label{eq7}
        \left\| \boldsymbol{\mu }_* \right\| -\left\| \boldsymbol{\mu }_Q \right\| \leqslant \left\| \boldsymbol{\mu }_*-\boldsymbol{\mu }_Q \right\| \leqslant v\sqrt{d \lambda _{\max}\left( \varSigma _{Q} \right)\log \left( \frac{2d}{\delta} \right)}
    \end{equation}
    holds with probability at least $1- \delta$.

    Now we decompose its regret. Let $\hat{A}_{s,t}$ be the optimal arm in round $t $ of instance $\boldsymbol{\mu }_{s}$, $A_{s,t}$ be the pulled arm in round $t$ by TS with misspecified prior $P_s=\mathcal{N}({\boldsymbol{\mu }}_{Q,s},v^2 \varSigma _* )$, and $\tilde{A}_{s,t}$ be the pulled arm in round $t$ by TS with correct prior $P_*=\mathcal{N}(\boldsymbol{\mu }_*,v^2 \varSigma _* )$. Denote $    \theta _{s,i}\left( t \right) =\boldsymbol{b}_i\left( t \right) ^T\boldsymbol{\mu }_s$, and  then
    $$
        \mathbb{E} \left[ \sum_{t=1}^n{\theta _{s,\hat{A}_{s,t}}\left( t \right) -\theta _{s,A_{s,t}}\left( t \right)} \middle| \:P_* \right] =R_{s,1}+R_{s,2},
    $$
    where
    $$
        R_{s,1}=\mathbb{E} \left[ \sum_{t=1}^n{\theta _{s,\hat{A}_{s,t}}\left( t \right) -\theta _{s,\tilde{A}_{s,t}}\left( t \right)}\middle| \:P_* \right] ,\: R_{s,2}=\mathbb{E} \left[ \sum_{t=1}^n{\theta _{s,\tilde{A}_{s,t}}\left( t \right) -\theta _{s,A_{s,t}}\left( t \right)}\:\middle| \:P_* \right] \:.
    $$
    The term $R_{s,1}$ is the regret of hypothetical TS that knows $P_*.$ This TS is introduced only for the purpose of analysis and is the optimal policy. The term $R_{s,2}$ is the difference in the expected $n$-round rewards of TS with priors  $P_s$  and $P_*$, and vanishes as the number of tasks $s$ increases.

    To bound $R_{s,1}$, we apply Lemma \ref{theo1} with $\delta=\frac{1}{n}$ and get
  $$
R_{s,1}\leqslant 2v\left( \sqrt{\frac{1}{\lambda _{\min}}}+\sqrt{\frac{n-1}{\vartheta}} \right) \sqrt{2\log \left( n \right)}+kv\sqrt{\frac{2}{\pi \lambda _{\min}}}.
$$
 To bound $R_{s,2}$, we apply Lemma \ref{theo2} with $\delta=\frac{1}{n}$ and get
$$
R_{s,2}\leqslant \frac{2k}{v}\left( \left\| \boldsymbol{\mu }_* \right\| +v\sqrt{\frac{2\log \left( n \right)}{\lambda _{\min}}} \right) \left( \sqrt{\frac{1}{\lambda _{\min}}}+\sqrt{\frac{n-1}{\vartheta}} \right) \sqrt{\frac{2\lambda _{\max}}{\pi}}\left\| \boldsymbol{\mu }_{Q,s}-\boldsymbol{\mu }_* \right\| +4kv\left( \sqrt{\frac{1}{2\pi \lambda _{\min}}}+\left\| \boldsymbol{\mu }_* \right\| \right) .
$$
    Let  $
        u_1\left( \delta \right) =\left\| \boldsymbol{\mu }_Q \right\| +v\sqrt{d \lambda _{\max}\left( \varSigma _{Q} \right)\log \left( \frac{2d}{\delta} \right)}
    $ and  according to \eqref{eq7}, we get
$$
R_{s,2}\leqslant \frac{2k}{v}\left( u_1\left( \delta \right) +v\sqrt{\frac{2\log \left( n \right)}{\lambda _{\min}}} \right) \left( \sqrt{\frac{1}{\lambda _{\min}}}+\sqrt{\frac{n-1}{\vartheta}} \right) \sqrt{\frac{2\lambda _{\max}}{\pi}}\left\| \boldsymbol{\mu }_{Q,s}-\boldsymbol{\mu }_* \right\| +4kv\left( \sqrt{\frac{1}{2\pi \lambda _{\min}}}+u_1\left( \delta \right) \right) .
$$
    holds with probability at least $1- \delta$.

    By Lemma \ref{theo3G}, we  have with probability at least $1-m\delta$ that
    $$
        \begin{aligned}
            \sum_{s=1}^m{\left\| \boldsymbol{\mu }_{Q,s}-\boldsymbol{\mu }_* \right\|} & \leqslant v\sqrt{d\log \left( \frac{2d}{\delta} \right)}\sum_{s=1}^m{\sqrt{\left(  \lambda _{\max}\left( \varSigma _{Q} \right)-\frac{2}{175\lambda _{\min}} \right) \left( \frac{7}{8} \right) ^{s-1}+\frac{2}{175\lambda _{\min}}}}
            \\
                                                                                       & \leqslant v\sqrt{d\log \left( \frac{2d}{\delta} \right)}\left( \sqrt{ \lambda _{\max}\left( \varSigma _{Q} \right)-\frac{2}{175\lambda _{\min}}}\sum_{s=1}^m{\sqrt{\left( \frac{7}{8} \right) ^{s-1}}}+\sum_{s=1}^m{\sqrt{\frac{2}{175\lambda _{\min}}}} \right)
            \\
                                                                                       & \leqslant 4v\log \left( m \right) \sqrt{d\left(  \lambda _{\max}\left( \varSigma _{Q} \right)-\frac{2}{175\lambda _{\min}} \right) \log \left( \frac{2d}{\delta} \right)}+vm\sqrt{d\frac{2}{175\lambda _{\min}}\log \left( \frac{2d}{\delta} \right)}.
        \end{aligned}
    $$

    Let
    $
        u_2\left( \delta \right) =\sqrt{d\left(  \lambda _{\max}\left( \varSigma _{Q} \right)-\frac{2}{175\lambda _{\min}} \right) \log \left( \frac{2d}{\delta} \right)},u_3\left( \delta \right) =\sqrt{d\frac{2}{175\lambda _{\min}}\log \left( \frac{2d}{\delta} \right)}
    $.
    Now we sum up our bounds on $R_{s,1}+R_{s,2}$ over  all tasks
    $s\in [m]$ and get
$$
\begin{aligned}
	\sum_{s=1}^m{R_{s,1}+R_{s,2}}&\leqslant 2mv\left( \sqrt{\frac{1}{\lambda _{\min}}}+\sqrt{\frac{n-1}{\vartheta}} \right) \sqrt{2\log \left( n \right)}+mkv\sqrt{\frac{2}{\pi \lambda _{\min}}}\\
	&\quad +2k\left( 4\log \left( m \right) u_2\left( \delta \right) +mu_3\left( \delta \right) \right) \left( u_1\left( \delta \right) +v\sqrt{\frac{2\log \left( n \right)}{\lambda _{\min}}} \right) \left( \sqrt{\frac{1}{\lambda _{\min}}}+\sqrt{\frac{n-1}{\vartheta}} \right) \sqrt{\frac{2\lambda _{\max}}{\pi}}\\
	&\quad +4mkv\left( \sqrt{\frac{1}{2\pi \lambda _{\min}}}+u_1\left( \delta \right) \right) .\\
\end{aligned}
$$


    This concludes our proof.

\end{proof}

\subsection{Proof of Lemma \ref{theo3}} \label{prooftheo3}
\begin{proof}
    According to the proof of \ref{theo3G}, we obtain
    $$
        \mathbb{P} \left( \left\| \boldsymbol{\mu }_*-\boldsymbol{\mu }_{Q,s} \right\| >\varepsilon \mid H_{1:s-1} \right) \leqslant2d\exp \left[ -\frac{\varepsilon ^2}{dv^2\left( \left(  \lambda _{\max}\left( \varSigma _{Q} \right)-\frac{2}{175\lambda _{\min}} \right) \left( \frac{7}{8} \right) ^{s-1}+\frac{2}{175\lambda _{\min}} \right)} \right].
    $$
    Choose
    $
        \varepsilon_s =v\sqrt{d\left( \left(  \lambda _{\max}\left( \varSigma _{Q} \right)-\frac{2}{175\lambda _{\min}} \right) \left( \frac{7}{8} \right) ^{s-1}+\frac{2}{175\lambda _{\min}} \right) \log \left( \frac{4d}{\delta} \right)}
    $
    and  we get that
    \begin{equation}
        \mathbb{P} \left( \left\| \boldsymbol{\mu }_*-\boldsymbol{\mu }_{Q,s} \right\| >\varepsilon_s \mid H_{1:s-1} \right) \leqslant \frac{\delta}{2}
    \end{equation}
    for any task $s$ and history $H_{1: s-1}$. It follows that
    $$
        \mathbb{P} \left( \bigcup_{s=1}^m{\left\{ \left\| \boldsymbol{\mu }_*-\boldsymbol{\mu }_{Q,s} \right\| >\varepsilon_s \right\}} \right) \leqslant \sum_{s=1}^m{\mathbb{P}}\left( \left\| \boldsymbol{\mu }_*-\boldsymbol{\mu }_{Q,s} \right\| >\varepsilon_s \right) =\sum_{s=1}^m{\mathbb{E}}\left[ \mathbb{P} \left( \left\| \boldsymbol{\mu }_*-\boldsymbol{\mu }_{Q,s} \right\| >\varepsilon_s \mid H_{1:s-1} \right) \right] \leqslant \frac{m\delta}{2}.
    $$

    Since $\hat{\boldsymbol{\mu}}_s|H_{1:s}$ is distributed identically to $\boldsymbol{\mu }_*|H_{1:s}$, we have from the same line of reasoning that
    $$
        \mathbb{P} \left( \bigcup_{s=1}^m{\left\{ \left\| \hat{\boldsymbol{\mu}}_s-\boldsymbol{\mu }_{Q,s} \right\| >\varepsilon_s \right\}} \right) \leqslant \frac{m\delta}{2}.
    $$

    Finally, we apply the triangle inequality and union bound,
    $$
        \mathbb{P} \left( \bigcup_{s=1}^m{\left\{ \left\| \hat{\boldsymbol{\mu}}_s-\boldsymbol{\mu }_* \right\| >2\varepsilon_s \right\}} \right) \leqslant \mathbb{P} \left( \bigcup_{s=1}^m{\left\{ \left\| \hat{\boldsymbol{\mu}}_s-\boldsymbol{\mu }_{Q,s} \right\| >\varepsilon_s \right\}} \right) +\mathbb{P} \left( \bigcup_{s=1}^m{\left\{ \left\| \boldsymbol{\mu }_*-\boldsymbol{\mu }_{Q,s} \right\| >\varepsilon_s \right\}} \right) \leqslant m\delta .
    $$
\end{proof}

\end{document}